\documentclass[11pt]{article}
\usepackage[bottom=1in,top=1in,left=1in,right=1in]{geometry}

\usepackage{setspace}
\usepackage{natbib}

\usepackage[utf8]{inputenc} 
\usepackage[T1]{fontenc}    
\usepackage{hyperref}       
\usepackage{url}            
\usepackage{booktabs}       
\usepackage{amsfonts}       
\usepackage{nicefrac}       
\usepackage{microtype}      
\usepackage{algorithm,algorithmic}
\usepackage{xcolor}         
\usepackage{multirow}
\usepackage{bbm}

\usepackage{amsmath,amscd,amsbsy,amssymb,latexsym,url,bm,amsthm}
\allowdisplaybreaks
\usepackage{cleveref}
\usepackage{verbatim}
\usepackage{color}
\usepackage{dsfont}
\newtheorem{remark}{Remark}

\newtheorem{assumption}{Assumption}
\newtheorem{lemma}{Lemma}

\usepackage{tikz}
\usetikzlibrary{shapes,shapes.misc,positioning}
\usetikzlibrary{patterns,arrows,decorations.pathreplacing}

\mathchardef\mhyphen="2D

\usepackage{hyperref}

\usepackage{times}
\usepackage{multirow}
 \usepackage{booktabs}
\usepackage{amsmath,amscd,amsbsy,amssymb,latexsym,url,bm,natbib}
\allowdisplaybreaks
\usepackage{cleveref}
\usepackage{verbatim}
\usepackage{color}
\usepackage{dsfont}
\usepackage{caption}
\usepackage{algorithm}
\usepackage{algorithmic}
\usepackage{threeparttable}
\usepackage{graphicx}
\usepackage{thmtools}
\usepackage{thm-restate}
\usepackage{mathtools}

\hypersetup{
  colorlinks   = true, 
  urlcolor     = blue, 
  linkcolor    = blue, 
  citecolor   = blue 
}

\title{
Elucidating Rectified Flow with Deterministic Sampler: Polynomial Discretization Complexity for Multi and One-step Models
}

\author{
  Ruofeng Yang$^{1}$, Zhaoyu Zhu$^{1}$, Bo Jiang$^{1}$, Cheng Chen$^{2}$, Shuai Li$^{*1}$ \\
  $^{1}$Shanghai Jiao Tong University, \text{\{wanshuiyin, zzy12345, bjiang, shuaili8\}@sjtu.edu.cn} \\
  $^{2}$East China Normal University, \text{chchen@sei.ecnu.edu.cn}
}
\date{}

\begin{document}
\maketitle
\begingroup
\renewcommand\thefootnote{*}
\footnotetext{Corresponding author: Shuai Li (\texttt{shuaili8@sjtu.edu.cn})}
\endgroup

\begin{abstract}
Recently, rectified flow (RF)-based models have achieved state-of-the-art performance in many areas for both the multi-step and one-step generation. However, only a few theoretical works analyze the discretization complexity of RF-based models. Existing works either focus on flow-based models with stochastic samplers or establish complexity results that exhibit exponential dependence on problem parameters. In this work, under the realistic bounded support assumption, we prove the first polynomial discretization complexity for multi-step and one-step RF-based models with a deterministic sampler simultaneously. For the multi-step setting, inspired by the predictor-corrector framework of diffusion models, we introduce a Langevin process as a corrector and show that RF-based models can achieve better polynomial discretization complexity than diffusion models.  To achieve this result, we conduct a detailed analysis of the RF-based model and explain why it is better than previous popular models, such as variance preserving (VP) and variance exploding (VE)-based models. Based on the observation of multi-step RF-based models, we further provide the first polynomial discretization complexity result for one-step RF-based models, improving upon prior results for one-step diffusion-based models.  These findings mark the first step toward theoretically understanding the impressive empirical performance of RF-based models in both multi-step and one-step generation.   
\end{abstract}
\section{Introduction}\label{sec:intro}
To improve the sampling speed,  rectified flow (RF), a flow-based generative model, has been proposed and has demonstrated impressive performance with fast sampling speed across various domains, including 2D, 3D, and video generation  \citep{liu2022flowfirst,esser2024scalingsd3, liu2023instaflow,li2024flowdreamer3dv1,go2024splatflow3dv2,wang2024frierenvideo}. Rectified flow aims to learn a vector $v(X_t,t)$ to generate target samples $X_1\sim q^*\in \mathbb{R}^d$ from a pure Gaussian $X_0=Z\sim \mathcal{N}(0,I_d)$.
To construct a tractable vector field, RF-based models first perform a linear interpolation between the target distribution and pure Gaussian \citep{liu2022flowfirst}:
\begin{align*}
    X_t = (1-t)Z+tX_1\,,t\in [0,1]\,.
\end{align*}
Let $(q_t)_{t\in [0,1]}$ denote the marginal distributions of the interpolated process $(X_t)_{t\in [0,1]}$. Based on this interpolation, the optimal vector field is defined as 
\begin{align*}
    v^*(t, X)=\mathbb{E}\left[X_1-Z \mid X_t=X\right]\,.
\end{align*}
With this optimal vector field, the following process (with a deterministic sampler) has the same marginal distribution $q_t$ with the above linear interpolation
\begin{align*}
   \mathrm{d} X_t=v^*\left(t, X_t\right)\mathrm{~d} t, \quad X_0\sim \mathcal{N}(0,I_d)\,.
\end{align*}
The above discussion indicates that if the RF models access the ground truth optimal vector field $v^*$ and run the continuous process, it can generate the target distribution $q^*$. However, the optimal vector field contains the data information and usually can not be directly calculated. Hence, current works adopt the flow-matching objective function and use a neural network $\hat{v}$ to approximate the optimal vector field \citep{liu2022flowfirst}. After that, the RF model discretizes the above continuous sampling process into $N$ intervals and runs the discretized process to generate samples.

When saying the RF model enjoys a faster sampling speed, researchers mean that the sampling complexity $N$ is small. Though the RF-based models enjoy a fast sampling speed in the application, only a few theoretical works analyze the sample complexity of RF models. As shown in \cref{tab:iteration results}, the current sample complexity results are much larger than the diffusion model results when considering the bounded support data. More specifically, \citet{gao2024convergence} provide the first convergence guarantee for the RF-bade model with a deterministic sampler under different target data              
      (mixture of Gaussian, Log-concave, and bounded support) and the guarantee has an exponential dependence on the problem parameters when considering the harder bounded support assumption (See details in \cref{subsec: detailed cal previous work}). \citet{silveri2024theoreticalflowkl} consider a flow-based method with a stochastic sampler (diffusion flow matching, DFM) and achieve a polynomial sample complexity. However, this result is much larger than the current polynomial results of diffusion models. Furthermore, their result adopts a stochastic sampler, and the empirical works use a deterministic sampler. Hence, the following question naturally opens:

\textit{Is it possible for the RF-based method with a deterministic sampler to achieve a comparable sample complexity with the diffusion-based method under the bounded support assumption?}

We note that the bounded support assumption is more realistic than the mixture of Gaussians and the log-concave distributions since it is naturally satisfied by the image datasets and admits the blow-up phenomenon of the vector field at the end of the generation process \citep{kim2021softepsilonl}. More specifically, under the bounded support assumption, $v^*$ goes to $+\infty$ as $t$ goes to $1$. For the blow-up phenomenon, both empirical works \citep{kim2021softepsilonl,song2020sde} and theoretical works \citep{gao2024convergence, silveri2024theoreticalflowkl,chen2022sampling} adopt the early stopping technique and stop the generation process at $T-\delta$ ($T=1$ for the flow-based models), where $\delta$ is the early stopping parameter. In this work, we also adopt this technique.

\subsection{Our Contribution.}
As shown in \cref{tab:iteration results}, for the first time, under the bounded support assumption, we achieve the polynomial discretization complexity for the multi-step and one-step RF-based models simultaneously. Since these results are better than the previous diffusion-based models, this work makes the first step toward explaining the great performance of RF-based models from a theoretical perspective. 
\paragraph{Analysis for the Property of Rectified Flow.}
To achieve the above results, we first analyze the properties of the RF-based model and compare it with three widely used models: (1) VP diffusion-based models, (2) VE diffusion-based models, and (3) DFM flow-based models. We mainly focus on the influence of early stopping parameter $\delta$ and diffusion time $T$ ($T=1$ for flow-based models), which corresponds to $\epsilon_{W_2}$ and $\epsilon_{\mathrm{TV}}$, respectively. For the early stopping parameter $\delta$, a small $\delta$ is needed to guarantee $W_2(q_{T-\delta}, q^*)\leq \epsilon_{W_2}$. For the VP and DFM models, due to the roughness of the Brownian motion (Details in \cref{subsec: detailed cal previous work}), the $\delta$ have the order of $\epsilon_{W_2}^2$, which is the source of bad $\epsilon_{W_2}$ dependence (\cref{tab:iteration results}). On the contrary, since the linear interpolation without stochastic, the RF-based models enjoy a $\delta$ with a better $\epsilon_{W_2}$ order.

Very recently, \citet{yang2025polynomial} show that the VE-based models also enjoy a $\delta=\epsilon_{W_2}$. However, as discussed in \citet{yang2024leveraging}, different from the other three models (VP, RF and DFM-based models), the VE-based models require a large polynomial diffusion time $T=1/\text{Poly}(\epsilon_{\mathrm{TV}})$ and introduce additional $\epsilon_{\mathrm{TV}}$ dependence. In conclusion, compared with VP, VE and DFM models, the RF-based model enjoys a better $\delta$ and $T$  simultaneously, which leads to a great sample complexity.

Based on the above observation, we (1) propose the vector perturbation lemma (lemma \ref{vperturbation}), which has a better dependence $1/\epsilon_{W_2}^8$ compared with previous $1/\epsilon_{W_2}^{12}$ one of VP-based models \citep{chen2023probability} and (2) eliminate the reverse beginning error of VE-based models \citep{yang2024leveraging}. Then, we achieve the following complexity results for the multi-step and one-step RF-based models.

\paragraph{Polynomial Sample Complexity for Multi-step RF with Deterministic Sampler.}
For the RF-based models with a deterministic sampler, we adopt the predictor-corrector framework and fully use different Lipschitz constants for the vector and corresponding score to design the algorithm (\cref{algo:algo1}). After that, by using the proposed vector perturbation lemma, we achieve the polynomial sample complexity $1/(\epsilon_{\mathrm{TV}}\epsilon_{W_2}^6)$, which is the first one for the RF-based models with a deterministic sampler and is better than previous $1/(\epsilon_{\mathrm{TV}}\epsilon_{W_2}^8)$ for VP diffusion-based models and $1/(\epsilon_{\mathrm{TV}}^{2}\epsilon_{W_2}^{7})$ for VE-based models.    

\paragraph{Better Discretization Complexity for One-step RF.}
As a byproduct of the vector perturbation lemma, we also provide the first discretization complexity $L_f/\epsilon_{W_2}^{3+1/a}$ for the one-step RF-based models, where $L_f$ is the Lipschitz constant for the one-step mapping function $f_{\theta}$ and $a$ is the parameter of discretization scheme. As shown in \cref{tab:iteration results}, this result is better than the results of the VP and VE-based one-step models and makes the first step to explain the great performance of InstaFlow \citep{liu2023instaflow} (Detailed discussion in \cref{sec: oen step}).

\begin{table*}[t]
    \centering
    \begin{tabular}{c|c|c|c|c}
    \hline 
    &Sampler &Model  & Complexity & Reference\\
    \hline
    \multirow{6}{*} {\begin{tabular}[c]{@{}l@{}}Multi-step\\  Generation\end{tabular}  } & Stochastic & DFM  & $1/\epsilon_{\mathrm{KL}}^2\epsilon_{W_2}^{16}$& \cite{silveri2024theoreticalflowkl}\\
    \cline{2-5} &\multirow{5}{*}  {Deterministic} & VP & $1/\epsilon_{\mathrm{TV}}\epsilon_{W_2}^8$& \citet{chen2023probability}\\ 
    \cline{3-5} & &VE &$e^{1/\epsilon_{W_1}}\text{Poly}(1/\epsilon_{W_1})$ &\citet{yang2024leveraging}\\ 
    \cline{3-5} & &VE &{\begin{tabular}[c]{@{}l@{}}$1/\epsilon_{\mathrm{TV}}^{2}\epsilon_{W_2}^7$ (Uniform)\\  $1/\epsilon_{\mathrm{TV}}^{7}\epsilon_{W_2}^6$ (Exponential)\end{tabular} } &\citet{yang2025polynomial}\\ 
    \cline{3-5} & &RF &$e^{1/\epsilon_{W_2}}\text{Poly}(1/\epsilon_{W_2})$ &\citet{gao2024convergence}\\ 
    \cline{3-5} & &RF &$1/\epsilon_{\mathrm{TV}}\epsilon_{W_2}^6$ &\ Theorem \ref{thm:main_theorem_reverse_PFODE}\\
    \hline

    \multirow{5}{*} {\begin{tabular}[c]{@{}l@{}}One-step\\  Generation\end{tabular} } & \multirow{4}{*} {Deterministic} & VP  & $L_f/\epsilon_{W_2}^7$& \cite{lyu2023convergenceconsistencyv1}\\
    \cline{3-5} &  & VP & $L_f^3/\epsilon_{W_1}$& \citet{li2024towardsconsistencyv2}\\
    \cline{3-5} 
    & & VP & \begin{tabular}[c]{@{}l@{}} $L_f^{2}L_{S,X}^2/\epsilon_{W_1}^2$ \\$L_f^2/\epsilon_{W_1}^{10}$ (*) \end{tabular} & \citet{doutheoryconsistency}\\ 
    \cline{3-5}&  & VE & $\qquad  \qquad \quad L_f/\epsilon_{W_2}^{3+2/a},a\in [1,+\infty)$& \citet{yang2025improved}\\
    \cline{3-5} & &RF &$\qquad  \qquad \quad L_f/\epsilon_{W_2}^{3+1/a},a\in [1,+\infty)$ &\ Theorem \ref{thm:instaflow_results}
    \\\hline
    \end{tabular}
    
    \caption{The complexity results for multi-step and one-step generative models under the bounded support assumption. The subscript of $\epsilon$ indicates the distance metric. For example, $1/(\epsilon_{\mathrm{TV}}\epsilon_{W_2}^{6})$ means the output is $\widetilde{O}(\epsilon_{\mathrm{TV}})$ close to $q_{1-\delta}$, which is $\epsilon_{W_2}$-close to $q^*$. The $W_2$ is stronger than the $W_1$ guarantee. We consider the total discretization complexity (including the predictor and corrector complexity) when considering the predictor-corrector algorithm. (*) means that we transform the results of previous works into the results under our setting. We present more detail in \cref{subsec: detailed cal previous work}. }
    \label{tab:iteration results}
\end{table*}

\section{Related Work}\label{sec:related work}
Since this work involves the analysis of multi-step and one-step RF-based models with a deterministic sampler, we discuss the current complexity results for flow-based models, diffusion models with a deterministic sampler, and one-step generative models, respectively.
\paragraph{Sample Complexity for Flow-based Models.}
\citet{gao2024convergence} analyze the RF-based method with a deterministic under different target data distributions: the mixture of Gaussians, the log-concave data, and bounded support data. In the first two data distributions, they achieve the polynomial sample complexity. However, these two distributions do not allow the blow-up phenomenon of the vector field, which can not explain the empirical observation \citep{kim2021softepsilonl}. On the contrary, the bounded support assumption allows this phenomenon and is satisfied by the image dataset, which indicates it is a more general assumption. However, when considering this assumption, \citet{gao2024convergence} only achieve an exponential guarantee depending on $\exp{(1/\epsilon_{W_2})}$, which is much worse than the results of diffusion-based models (\cref{tab:iteration results}). When considering the RF-based method with a stochastic sampler, \citet{silveri2024theoreticalflowkl} achieve a polynomial sample complexity. However, their result is still larger than that of the diffusion-based method.

\paragraph{Sample Complexity for Diffusion models with a Deterministic Sampler.}
Since generative models with a deterministic sampler enjoy a fast sampling speed, the sample complexity for a deterministic sampler has been widely analyzed. However, these works either have an exponential dependence \citep{chen2023restoration,yang2024leveraging} or require a strong assumption on the data or the score function. More specifically, 
\citet{gao2024convergencediffusion} assume the target data is log-concave and achieve polynomial sample complexity.
Assuming the Jacobian matrix of the matrix is accurate enough, \citet{li2023towards} and \citet{li2024sharp} achieve nearly optimal sample complexity. However, the log-concave data is not usually satisfied by the real-world data, and the score-matching objective function for diffusion models can not guarantee that the Jacobian matrix is accurate enough. Without these strong assumptions, \citet{chen2023probability} introduce a Langevin process as the corrector for the deterministic sampler, use the predictor-corrector framework and achieve a polynomial sample complexity. Using a similar framework, \citet{yinuochen2024accelerating} and \citet{yang2025polynomial} achieve the polynomial sample complexity for the parallel sampling algorithm and VE-based models, respectively.
We note that the predictor-corrector algorithm is adopted by the empirical work  \citep{song2020sde,zhang2024collapse} for the deterministic sampler. In this work, we also adopt this framework to show the advantage of RF-based models compared with diffusion-based models (VP and VE-based models).

\paragraph{Discretization Complexity for One-step Generation Models.} Recently, some works have analyzed the discretization complexity of one-step diffusion-based models \citep{lyu2023convergenceconsistencyv1, li2024towardsconsistencyv2, doutheoryconsistency,yang2025improved}. However, as shown in \cref{tab:iteration results}, the discretization complexity of these works either achieve a weaker $W_1$ guarantee (compared with $W_2$ guarantee) \citep{li2024towardsconsistencyv2,doutheoryconsistency} or have a large $W_2$ dependence \citep{lyu2023convergenceconsistencyv1,yang2025improved}. In this work, we achieve a better result for one-step RF-based models compared with the VP and VE-based one-step generation models (\cref{thm:instaflow_results}).
\section{Preliminaries}\label{sec:intro of RF}
The goal of rectified flow is to generate samples from the target data distribution  $X_1\sim q^{*}\in \mathbb{R}^d$ starting from a standard Gaussian initialization  $X_0=Z \sim \mathcal{N}(0,I_d)$ though a vector field $v(t,X_t)$ and a deterministic sampler. To construct a tractable vector field, the RF-based model first performs a linear interpolation between the target distribution and the pure Gaussian, as proposed in  \citet{liu2022flowfirst}:
\begin{align}\label{eq:linear interpolation}
    X_t = (1-t)Z+tX_1\,,t\in [0,1]\,.
\end{align}
Let $(q_t)_{t\in [0,1]}$ be the marginal distribution of $(X_t)_{t\in [0,1]}$. Given a linear Gaussian stochastic interpolation, the optimal vector field has the following form 
\begin{align*}
    v^*(t,X)=\mathbb{E}\left[X_1-Z |X_t=X\right], \quad(t, X) \in[0,1] \times \mathbb{R}^d\,.
\end{align*}
With the above $v^*$, the following continuous ODE has the same marginal distribution as $X_t$:
\begin{align}\label{eq:the PFODE}
   \mathrm{d} X_t=v^*\left(t, X_t\right)\mathrm{~d} t, \quad X_0\sim \mathcal{N}(0,I_d)\,.
\end{align}
By running this continuous process starting from $t=0$ and ending at $t=1$, the generated distribution is $q^*$. However, since the optimal vector field $v^*$ usually corresponds to the unknown target distribution $q^*$, previous works use the flow matching technique to learn an approximate vector field $\hat{v}$:
\begin{align*}
    \hat{v} \in \underset{v\in \mathcal{F}}{\arg \min }\left\{\mathcal{L}(v):=\mathbb{E}_{t \sim U[0, 1-\delta]} \mathbb{E}_{X_1 \sim q^*, Z \sim \gamma_d}\left\|v\left(t, X_{t}\right)-(X_1-Z)\right\|_2^2\right\}\,,
\end{align*}
where $\mathcal{F}$ is a pre-defined function class.
As shown in \cref{def: ground truth vector}, $v^*(t,\cdot)$ blows up at the end of the generated process. Hence, we introduce an early stopping parameter $\delta$ and run the generated process in $t\in [0,1-\delta]$. Since the RF model can not run the continuous process, the model discretizes this process into $N$ intervals. Let $t_0=0<t_1<\cdots<t_N=1-\delta$ be the discretization points. With the approximated vector field $\hat{v}$, the RF model runs the following discrete process to generate samples for $n\in [1,N]$:
\begin{align}\label{def:predictor}
    \mathrm{d} \hat{X}_t=\hat{v}\left(t_{n-1}, \hat{X}_{t_{n-1}}\right)\mathrm{~d} t\,,\hat{X}_0\sim\mathcal{N}(0,I_d)\,, t \in\left[t_{n-1}, t_n\right]\,.
\end{align}

Since the above process predicts a cleaner sample from a noisy one, we named it the predictor, which is also used in the previous work \citep{song2020sde,chen2023probability,yinuochen2024accelerating}. In this work, we use the uniform with predictor stepsize $h_{\mathrm{pred}}$, which means $t_n-t_{n-1}=h_{\mathrm{pred}}$ for $n\in [1,N]$ and the predictor discretization complexity is $N=(1-\delta)/h_{\mathrm{pred}}$.

\subsection{The Underdamped Langevin Diffusion}\label{subsec:introduce of ULC}
From the empirical perspective, \citet{song2020sde} first propose the predictor-corrector framework and achieve great performance. \citet{zhang2024collapse} show that the modal collapse phenomenon of the deterministic predictor (sampler) can be reduced with an additional stochastic corrector. From the theoretical perspective, several theoretical works on deterministic samplers have introduced the underdamped Langevin process to achieve polynomial discretization complexity \citep{chen2023probability, yinuochen2024accelerating, yang2025improved}. This work also incorporates an underdamped Langevin corrector, enabling us to achieve polynomial complexity for rectified flow with a deterministic sampler.

Since the underdamped Langevin corrector works with a fixed $t$, we use $U_t$ as a shorthand for the potential $-\log q_t$. To distinguish from the predictor stage, we denote $m$ by the time of the corrector. Let $\rho>0$ be the friction parameter. The ULD is a stochastic process $\left(z_m, v_m\right)_{m \geq 0}$ over $\mathbb{R}^d \times \mathbb{R}^d$ given by 
\begin{align}\label{eq:ULD}
&\mathrm{d} z_m=v_m \mathrm{~d} m\notag\\
    &\mathrm{~d} v_m=-\left(\nabla U_t\left(z_m\right)+\rho v_m\right) \mathrm{d} m+\sqrt{2 \rho} \mathrm{d} B_m\,.
\end{align}
We define by $\boldsymbol{q}:=q \otimes \gamma^d$ with $\gamma^d=\mathcal{N}(0,I)$. Then, the stationary distribution of this process is $\boldsymbol{q}_t=q_t \otimes \gamma^d$. Since we can not run the continuous process, we also discretize the above corrector process with stepsize $h_{\mathrm{corr}}$. The results $\left(\hat{z}_m, \hat{v}_m\right)_{m \geq 0}$ is given by 
\begin{align}\label{eq:ULMC}
&\mathrm{d} \hat{z}_m=\hat{v}_m \mathrm{~d} m\notag\\
&\mathrm{~d} \hat{v}_m=\left(s_t(\hat{z}_{\lfloor \frac{m}{h_{\text{corr}}} \rfloor h_{\text{corr}}})-\rho \hat{v}_m\right) \mathrm{d} m+\sqrt{2 \rho} \mathrm{d} B_m\,,
\end{align}
where $s_t(X,t)$ is an approximation of ground truth score function $\nabla\log q_t(X)$. Instead of learning an additional score function, we obtain it by fully using the approximated vector field $\hat{v}$. As shown in \citet{gao2023gaussianflow}, the ground truth vector field has the following form 
\begin{align}\label{def: ground truth vector}
    v^*(t,X) = -\frac{1}{1-t}X+\frac{1}{1-t}\mathbb{E}[X_1|X_t=X]\,.
\end{align}
The ground truth score function is equal to 
\begin{align}\label{score velocity transform}
        \nabla \log q_t(X) = -\frac{1}{(1-t)^2}X+\frac{t}{(1-t)^2}\mathbb{E}[X_1|X_t=X]\,,
\end{align}
which indicates $\nabla \log q_t(X)=-\frac{1}{1-t}X+\frac{t}{(1-t)}v^*(t,X)$. Then, we obtain an approximated score $s_t$ by using the approximated vector field and a similar relationship
\begin{align}\label{eq:transfer between score and vector}
    s_t(X)=-\frac{1}{1-t}X+\frac{t}{(1-t)}\hat{v}(t,X)
\end{align}

\paragraph{Notation.} For $x\in \mathbb{R}^d$ and $A\in \mathbb{R}^{d\times d}$, we denote by $\|x\|$ and $\|A\|$ the $L_2$ norm for vector and matrix. We denote by $W_1$ and $W_2$ the Wasserstein distances of order one and two, respectively. Note that $W_1$ guarantee is weaker than $W_2$ guarantee since $W_1(p,q)\leq W_2(p,q)$. We also define the push-forward operator $\sharp$, which is associated with a measurable map $f: \mathcal{M}^{\prime} \rightarrow \mathcal{N}$ and used in \cref{sec: oen step}. For any measure $\mu$ over $\mathcal{M}^{\prime}$, we define the push-forward measure $f \sharp \mu$ over $\mathcal{N}$ by: $f \sharp \mu(A)=$ $\mu\left(f^{-1}(A)\right)$, for any $A$ be measurable set in $\mathcal{N}$. 

Since this work involves the discretization complexity of multi-step and one-step RF models, we provide a complete notation in \cref{sec:app notation}.
\section{Guarantee for RF with Deterministic Sampler and Langevin Corrector}\label{sec: multistep}
This section provides the first polynomial sample complexity for the RF-based models with a deterministic sampler. Section \ref{sec: algorithm} introduces the algorithm under the predictor-corrector framework, which fully uses the Lipschitz constant of the optimal vector field and corresponding score function. Then, we provide the sample complexity for the proposed algorithm in \cref{subsec:theory for multi-step}.
\subsection{Algorithm}\label{sec: algorithm}
As a beginning, we first introduce the assumption on the target data distribution $q^*$, which is useful in the design of the algorithm.
\begin{assumption}\label{ass: Manifold assumption}
$q^*$ is supported on a compact set $\mathcal{M}$, $0 \in \mathcal{M}$ and $R=\sup \{\|x-y\|: x, y \in \mathcal{M}\}\ge 1$. 
\end{assumption}
As discussed at the end of \cref{sec:intro}, this assumption is realistic since it allows the blow-up phenomenon and is satisfied by image datasets. Furthermore, it has been widely used in many theoretical works \citep{de2022convergence, chen2022sampling,yang2024leveraging,yang2025polynomial}.

After that, we define some Lipschitz constants for the vector field and the corresponding score function. By simple algebra (\cref{subsec:the property of optimal vector}), under \cref{ass: Manifold assumption}, we have the following inequality:
\begin{align*}
    \|\nabla v^*(t,X)\|\leq R^2/\delta^3\,, \text{and }  \|\partial_t v^*(t,X)\|\leq R^2\sqrt{R^2 \vee d}/\delta^4\,, \forall t\in [0,1-\delta], X\in \mathbb{R}^d\,. 
\end{align*}
Hence, the Lipschitz constant of $v^*(t,X)$ respect to $t$ is defined by $L_{v,t}=R^2\sqrt{R^2 \vee d}/\delta^4$ and the Lipschitz constant of $v^*(t,X)$ respect to $X$ is defined by $L_{v,X}=R^2/\delta^3$. The Lipschitz constant of the corresponding ground score function $\nabla \log q_t(X)$ with respect to $X$ is $L_{S,X}=R^2/\delta^4$ due to \cref{eq:transfer between score and vector}.

With the above Lipschitz constant, we describe our algorithm, which switches between the predictor and corrector stages.
Let $N'$ be the number of the predictor stage. For the predictor stage, the predictor runs the deterministic sampling process for $T_{\mathrm{pred}}=1/L_{v,X}$ with the predictor stepsize $h_{\mathrm{pred}}$. Then, the total sampling time $T$ is defined by $T=N'T_{\mathrm{pred}}=1-\delta$.

After the predictor stage, the corrector runs the ULD process for $T_{\mathrm{corr}}=1/\sqrt{L_{S,X}}$ with the corrector stepsize $h_{\mathrm{corr}}$. The formal algorithm is presented in \cref{algo:algo1}.

\begin{algorithm*}[!thb]
\caption{RF with deterministic sampler and Langevin process}\label{algo:algo1}
\begin{algorithmic}[1]
\STATE \textbf{Input:} Approximated vector field $\hat{v}$, predictor stepsize $h_{\mathrm{pred}}$, corrector stepsize $h_{\mathrm{corr}}$.
\STATE \textbf{Initialization:}
Draw $\hat{X}_0$ from $\mathcal{N}(0,I_d)$.
\FOR{$n=0,1,\ldots,N'-1$}
    \STATE \textbf{Predictor.} Starting from $\hat{X}_{k/L_{v,X}}$, run the Flow (\ref{def:predictor}) from time $n/L_{v,X}$ to $(n+1)/L_{v,X}$ with stepsize $h_{\text {pred}}$  to obtain $\hat{X}_{(n+1)/L_{v,X}}^{\prime}$.
    \STATE \textbf{Corrector.} Starting from $\hat{X}_{(n+1)/L_{v,X}}^{\prime}$, run ULMC (\ref{eq:ULMC}) for total time $1/\sqrt{L_{S,X}}$ with stepsize $h_{\text {corr}}$ and score $s_{(n+1)/L_{v,X}}$ (obtain by \cref{eq:transfer between score and vector} with $\hat{v}$) to obtain $\hat{X}_{(n+1)/L_{v,X}}$.
\ENDFOR

\RETURN $\hat{X}_{1-\delta}$
\end{algorithmic}
\end{algorithm*}

\begin{remark}\label{remark: discussion on bounded support data}
When considering DFM-based models, \citet{silveri2024theoreticalflowkl} require the eighth-order moment of the target data to be bounded, which is slightly weaker than our assumption. We note that \cref{ass: Manifold assumption} is used to guarantee the Lipschitz constant of the vector field is bounded, which is useful in designing the algorithm and obtaining the vector perturbation lemma. It is an interesting direction to relax our assumption on target data, and we leave it as future work. 
\end{remark}
\subsection{Theoretical Results}\label{subsec:theory for multi-step}
Before showing the sample complexity of \cref{algo:algo1}, we first introduce some standard assumptions on the approximated vector field.
We first assume the learned vector field is accurate enough, which is widely used in the diffusion-based models \citep{de2022convergence,chen2022sampling,benton2023linear,yang2024leveraging} and the flow-based models \citep{silveri2024theoreticalflowkl}.
\begin{assumption}\label{ass:score approxiamation assumption}
    For the any time $t$, we have $\mathbb{E}_{X_t\sim q_t}[\|v^*(t,X_t)-\hat{v}(t,X_t)\|^2] \leq \varepsilon_{sc}^2$. 
\end{assumption}
We also assume the approximated vector field is $L_{v,X}$-Lipschitz.
\begin{assumption}\label{hatvlip}
$\hat{v}(t,X_t)$ is $L_{v,X}$-Lipschitz in every point used in our algorithm.
\end{assumption}
This assumption is standard for the predictor-corrector algorithm \citep{chen2023probability,yinuochen2024accelerating, yang2025polynomial}. Furthermore, $L_{v,X}=R^2/\delta^3$ is a large constant when the early stopping parameter $\delta$ is small. Hence, this assumption can be satisfied by using the truncation technique.

Let $p_{1-\delta}^{\mathrm{ULMC}}$ be the output of \cref{algo:algo1}. Then, with the above assumptions, we achieve the first polynomial sample complexity for the RF-based models with a deterministic sampler.

\begin{restatable}{theorem}{thmmaintheoremreversePFODE}\label{thm:main_theorem_reverse_PFODE}
Assuming \cref{ass: Manifold assumption}, \ref{ass:score approxiamation assumption}, \ref{hatvlip} holds. Then, we have the following bound
\begin{align*}
     \operatorname{TV}(p_{1-\delta}^{\operatorname{ULMC}},q_{1-\delta})\lesssim \frac{R^3(R\vee \sqrt{d})}{\delta^6}h_{\mathrm{pred}}+\frac{R^3}{\delta^5}h_{\mathrm{corr}}+\frac{\varepsilon_{sc}}{\delta^2}\,.
\end{align*}
Furthermore, by choosing $ \delta=\epsilon_{W_2}/(\sqrt{d}+ R), h_{\mathrm{corr}}\leq \delta^5\epsilon_{\operatorname{TV}}/R^3, h_{\mathrm{pred}}\leq \delta^{6}\epsilon_{\operatorname{TV}}/(R^3(R\vee\sqrt{d}))$, and assume $\varepsilon_{sc}\leq\delta^2\epsilon_{W_2}$, the output is $\epsilon_{\operatorname{TV}}$-close to $q_{1-\delta}$, which is $\epsilon_{W_2}$-close to $q^*$, with total iteration complexity
\begin{align*}
     \max\left\{\frac{1}{h_{\mathrm{pred}}}, \frac{N'}{\sqrt{L_{S,X}}h_{\mathrm{corr}}}\right\}\leq \Tilde{O}\left(\frac{R^3(R\vee\sqrt{d})}{\epsilon_{\operatorname{TV}}\epsilon_{W_2}^{6}}\right)\,.
\end{align*}
\end{restatable}
This result significantly improves the exponential dependence results for the RF-based models \citep{gao2024convergence} and achieves the first polynomial sample complexity. Furthermore, as shown in \cref{tab:iteration results}, the results is much better than $1/\epsilon_{\mathrm{KL}}^2\epsilon_{W_2}^{16}$ result of the DFM models, the $1/\epsilon_{\mathrm{TV}}\epsilon_{W_2}^{8}$ result for the VP-based diffusion models and the $1/\epsilon_{\mathrm{TV}}^{2}\epsilon_{W_2}^{7}$ result for the VE-based diffusion models. The great performance of RF-based models is due to the better dependence on the early stopping parameter $\delta$ and diffusion time $T$, which are discussed in the following part.
\subsection{Technique Novelty}\label{subsec:technique novelty}
This section first briefly introduces the proof sketch of the sample complexity results for the algorithm. Then, we propose the vector perturbation lemma, which is the core of the great performance of
\cref{thm:main_theorem_reverse_PFODE} and \ref{thm:instaflow_results}. Finally, we discuss the advantages of the RF-based models compared with previous popular models (VP, VE and DFM-based models).
\paragraph{Proof Sketch.} The proof sketch of \cref{thm:main_theorem_reverse_PFODE} is similar with \citet{chen2023probability}. The predictor stage of the algorithm runs the discretized sampler (\cref{def:predictor}) to push the time from $n/L_{v,X}$ to $(n+1)/L_{v,X}$, which is bounded by $\|X_t-\hat{X}_t\|^2$ in $W_2$ distance. However, as shown in lemma 4 of \citet{gao2024convergence}, if trivially integrating from $0$ to $1-\delta$, there will be an exponential term $e^{L_{v,X}}$ , where $L_{v,X}=R^2/\delta^3$ is a large constant.

To avoid exponential dependence, the algorithm introduces the corrector stage to inject suitable noise, which decouples each predictor stage and allows the use of the data processing inequality. With this corrector, the exponential $L_{v,X}$ dependence can be replaced with the polynomial one.

\paragraph{Vector Perturbation Lemma.}
Since our technique novelty mainly lies in the predictor stage, we first show how to bound one predictor error $\|X_t-\hat{X}_t\|^2$ starting from the same distribution. Using the definition of $X_t$ (\cref{eq:the PFODE}) and $\hat{X}_t$ (\cref{def:predictor}), the following equation holds
\begin{align*}
\partial_t \|X_t - \hat{X}_t\|^2 &=2\Big\langle X_t-\hat{X}_t, \frac{\mathrm{~d} X}{\mathrm{~d} t}-\frac{\mathrm{d} \hat{X}}{\mathrm{~d} t}\Big\rangle\\&=  \langle X_t - \hat{X}_t, v^*(t,X_t)-\hat{v}(t,\hat{X}_t)\rangle \\
&\leq \frac{1}{h} \|X_t - \hat{X}_t\|^2 + h \|v^*(t,X_t)-\hat{v}(t,\hat{X}_t)\|^2\,.
\end{align*}
By Gr\"{o}nwall's inequality, we know that
\begin{align*}
    \mathbb{E}[\|X_{t+h} - \hat{X}_{t+h}\|^2] \leq \exp( \frac{1}{h}\cdot h) \int_{t}^{t+h} h \mathbb{E}[\|v^*(t,X_t)-\hat{v}(t,\hat{X}_t)\|^2]  \mathrm{~d}t\,,
\end{align*}
which is corresponds to $\underset{X_t \sim q_t}{\mathbb{E}}\|\partial_t v^*\left(t, x_t\right)\|^2$. Therefore, we conduct a refined analysis of this term and propose the following vector perturbation lemma.

\begin{restatable}{lemma}{vperturbation}\label{vperturbation}
    Assume \cref{ass: Manifold assumption} holds. Then, we have the following bound 
    \begin{align*}
        \underset{X_t\sim q_t}{\mathbb{E}}\|\partial_tv^*(t,X_t)\|^2\lesssim R^4(R^2 \vee d)/(1-t)^8\leq R^4(R^2 \vee d)/\delta^8=L_{v,t}^2\,.
    \end{align*}
\end{restatable}
This lemma is one of our core contributions, which is used to control the predictor error in the multi-step RF method (\cref{thm:main_theorem_reverse_PFODE}) and one-step discretization error in the one-step RF method (\cref{thm:instaflow_results}). In the following paragraph, we discuss three widely analyzed models: (1) VP-based diffusion models, (2) VE-based diffusion models, and (3) DFM models, and show why the RF-based model with a deterministic sampler can achieve a better result. We mainly focus on two points: the early stopping parameter $\delta$ and diffusion time $T$.
\paragraph{Better Dependence on $\delta$ Compared with VP and DFM Models.}
To avoid the blow-up phenomenon at the end of the generation process \citep{kim2021softepsilonl}, current theoretical works adopt the early stopping technique \citep{gao2024convergence, chen2023probability, yang2024leveraging, yang2025polynomial}. More specifically, we need to determine an early stopping parameter $\delta$ to guarantee $W_2(q_{T-\delta},q^*)\leq \epsilon_{W_2}$ (recall that $q_T=q^*$) \footnote{For the diffusion-based models, $T$ is the diffusion time, and for the RF-based models, $T=1$.}. For the VP-based diffusion models \citep{chen2023probability} and diffusion flow matching (DFM), due to
the roughness of the Brownian motion, $\delta$ has order $\delta=\epsilon_{W_2}^2$ (Here, we ignore the dependence of $R,d$ for simplicity). On the contrary, due to the linear interpolation without stochastic, $\delta=\epsilon_{W_2}$ for our RF-based models.  For the VP-based models, the bad $\delta$ directly leads to a worse score perturbation lemma compared with \cref{vperturbation} and then a worse sample complexity in \cref{tab:iteration results}. More specifically, Lemma 1 of \citet{chen2023probability} shows that
\begin{align*}
\mathbb{E}\left[\left\|\partial_t \nabla \log q_t\left(X_t\right)\right\|^2\right] \leq L_{S,X}^2 d\left(L_{S,X}+\frac{1}{\delta}\right)\,,
\end{align*}
where $L_{S,X}=R^2/\delta^2$ for the VP-based diffusion models. It is clear that when considering $\delta$, $1/\delta^6$ dependence of \citet{chen2023probability} is better than the $1/\delta^8$ dependence of lemma \ref{vperturbation}. However, with $\delta=\varepsilon_{W_2}^2 /(\sqrt{d}(R \vee \sqrt{d}))$ \citep{chen2022sampling}, the right hand of the above inequality has the order of $1/\epsilon_{W_2}^{12}$, which is much worse than the $1/\epsilon_{W_2}^{8}$ order of \cref{vperturbation}.  

For the VE-based models, \citet{yang2025polynomial} show that $\delta$ has the order of $\epsilon_{W_2}$ with a specific noise schedule. However, as shown in the following paragraph, the VE-based diffusion models require a large diffusion time, introducing additional $\epsilon_{\mathrm{TV}}$ dependence. 
\paragraph{Better Dependence on Diffusion Time Compared with VE-based Models.}
Instead of directly linear interpolation between the target data $q^*$ and pure Gaussian, the diffusion-based models adopt the diffusion process to gradually convert data to Gaussian noise during the diffusion time $T$. Hence, for the diffusion-based models, there is an additional reverse beginning error term, which measures the distance between the noisy target data and pure Gaussian. As shown in \citet{yang2024leveraging}, the reverse beginning error for the VP-based models has the order of $\exp{(-T)}$, and we require $T=\log(1/\epsilon_{\mathrm{TV}})$ to guarantee this term is smaller than $\epsilon_{\mathrm{TV}}$. Since $T$ is a logarithmic term, it will not heavily affect the sample complexity for the VP-based models. However, for the VE-based models, the reverse beginning error has an order of $1/\text{Poly}(T)$ and leads to a polynomial $T=1/\text{Poly}(\epsilon_{\mathrm{TV}})$, which will introduce additional $\epsilon_{\mathrm{TV}}$ dependence in the final sample complexity. On the contrary, $T=1$ for the RF-based models and will not influence the discretization complexity.


\section{The Discretization Complexity for One-step Rectified Flow}\label{sec: oen step}
Due to the linear interpolation operation, the RF-based methods enjoy a linear trajectory, which indicates it is possible to obtain a one-step generative model by using the distillation method when given a pre-trained vector field $\hat{v}_{\phi}$. Let 
$f^{v^*}: \mathbb{R}^d \times \mathbb{R}^{+} \rightarrow \mathbb{R}^d$ be the associated backward mapping of the generation ODE process with the optimal vector field. Then, we know that:
\begin{align*}
f^{v^*}\left(X_{t}, t\right)=X_{1-\delta}\,,\forall t\in [0,1-\delta]\,,
\end{align*}
which indicates the distillation goal is to find a one-step mapping function $f_{\theta}$ to approximate $f^{v^{*}}$.

Intuitively, the trajectory of a great enough multi-step model is linear, and a trivial distillation method is enough.
Based on this idea,  \citet{liu2023instaflow} propose InstaFlow with the following objective:
\begin{align*}
    \mathcal{L}_{\mathrm{Insta}}(\theta; \phi) = \mathbb{E}_{X_0\sim \mathcal{N}(0,I)}\left[\mathbb{D}(\mathrm{ODE}[\hat{v}_{\phi}], f_{\theta}(X_0))\right]\,,
\end{align*}
where $\mathrm{ODE}(\hat{v}_{\phi})$ is the output of the discretized sampling process with a given pre-trained $\hat{v}_{\phi}$ (\cref{def:predictor}) and $\mathbb{D}$ is a similarity loss (LPIPS loss is used in \citet{liu2023instaflow}). However, \citet{wang2024rectified} show that the consistency distillation method \citep{song2020sde} leads to significantly faster training and better performance compared with the trivial distillation since InstaFlow imposes a stronger constraint than self-consistency. The intuition of consistency distillation is mainly based on the following property of $f^{v^*}$:  
\begin{equation}\label{eq: the property of consistency function}
    \begin{aligned}
&f^{v^*}\left(X_{t_1}, t_1\right)=f^{v^*}\left(X_{t_2}, t_2\right), \forall\, 0 \leq t_1, t_2 \leq 1-\delta\,,\\
&f^{v^*}(X, 1-\delta)=X, \,\forall\,X \in \mathbb{R}^d\,.
\end{aligned}
\end{equation}
Let $0=t_0\leq t_1\leq \cdots\leq t_K=1-\delta$ be the discretization points used in the training process of the one-step generative models. Based on this property, \citet{song2023consistency}  propose the following consistency distillation objective function:
\begin{align}\label{eq:CD objective}
&\mathcal{L}_{\mathrm{Consistency}}^K\left(\theta, \theta^{-} ; \phi\right):=\mathbb{E}_{X_1}\left[\mathbb{E}_{X_{t_k}|X_1}\left\|f_{\theta}(X_{t_{k}}, t_{k})-f_{\theta^{-}}(\hat{X}_{t_{k+1}}^{\phi}, t_{k+1})\right\|_2^2\right]\,,
\end{align}
where $X_{t_k}|X_1$ is calculated by $X_{t_k} = t_kX_1+(1-t_k)Z$, where $Z$ is the standard Gaussian noise. To avoid the strong constraint on the output of the neural network, \citet{song2023consistency} parameter $f_{\theta}$ with the following form 
\begin{align*}
f_{\theta}(X, t)=\left\{\begin{array}{ll}Y & t=1-\delta \\ F_{\theta}(Y, t) & t \in[0, 1-\delta)\end{array}\right.\,,
\end{align*}
where $F_{\theta}(X,t)$ be a free-form deep neural network.
Furthermore, \citet{song2023consistency} adopt the exponential moving average strategy $\theta^{-}=\operatorname{stopgrad}\left(\mu \theta^{-}+(1-\mu) \theta\right)$ to make the training process more stable. For $\hat{X}_{t_{k+1}}^\phi$, it is  the output of running one step flow from $t_{k}$ to $t_{k+1}$ using the initial distribution $X_{t_k}\sim q_{t_k}$ and approximated velocity field $\hat{v}_{\phi}$:
\begin{align}\label{eq:onestep flow}
\hat{X}_{t_{k+1}}^\phi=X_{t_k}+\hat{v}_{\phi}\left(t_k,X_k\right)(t_{k+1}-t_k)
\end{align}

\paragraph{Discretization Complexity.} One important problem for the training process of one-step models is the choice of discretization number $K$ (discretization complexity). If choosing a small $K$, the one-step error from $X_{t_k}$ to $X_{t_{k+1}}^{\phi}$ will be large and introduce hardness in the training process. When $K$ is large, the training process is time-consuming. Hence, it is necessary for one-step RF models to determine the discretization complexity $K$ in the training process.  

 The core challenge in obtaining the discretization complexity $K$ is the one-step predictor error introduced by \cref{eq:onestep flow}. As discussed in \cref{sec: multistep}, lemma \ref{vperturbation} has a better dependence compared with the previous perturbation lemma and can be directly used in the analysis of one-step models. As a result, we achieve the first discretization complexity result for one-step RF-based models, which is also better than previous one-step diffusion-based models. 

Before presenting our main result, we introduce two additional assumptions regarding the one-step mapping function $f_{\theta}$. We first assume that after a one-step predictor, the output of $f_{\theta}$ is still close, which is widely used in the previous works \citep{lyu2023convergenceconsistencyv1,li2024towardsconsistencyv2,yang2025improved}.
\begin{assumption}\label{ass:onestepfmatcherror}
There exists a constant \( \varepsilon_{\mathrm{cm}} \) such that for any \( k \in [K] \) and \( X_{t_k} \sim q_{t_k} \):
\[
\mathbb{E} \left[ \left\| f_\theta (X_{t_k}, t_k) - f_\theta (\hat{X}_{t_{k+1}}^\phi, t_{k+1}) \right\|_2^2 \right] \leq \epsilon_{cm}^2 \left( t_{k+1} - t_k \right)^2.
\]
\end{assumption}
Since this assumption has the same form as the consistency distillation objective function (\cref{eq:CD objective}), it can be satisfied as long as the loss function is small.

We also assume the learned one-step mapping function is $L_f$-Lipschitz.
\begin{assumption}\label{ass:flip}
    $f_\theta (X_{t}, t)$ is $L_f$-Lipschitz for any t in $[0,1-\delta]$.
\end{assumption}
This assumption is a standard assumption and has been widely used in all previous works on one-step generative models \citep{song2023consistency,lyu2023convergenceconsistencyv1,li2024towardsconsistencyv2,doutheoryconsistency}.

Let $f_{\theta,0}$ be the learned one-step mapping function at $t=0$ and $h_k=t_k-t_{k-1}$ be the stepsize at the $k$-th step. As shown in \citet{yang2025improved}, a gradually decay EDM stepsize \citep{karras2022elucidatingvesde} in the training phase is helpful in improving the discretization complexity of one-step generation models. Hence, we also adopt the EDM stpesize
\begin{align*}
1-t_k=(\delta+k h)^a \text { and } h=\left(1-\delta\right) /  K
\end{align*}
 in the discretization analysis of one-step RF models, where $a\in [1,+\infty)$ is the parameter of EDM stepsize. When $a$ is equal to $1$, the stepsize is a uniform one. When $a$ goes to $+\infty$, the stepsize becomes the exponential-decay stepsize.

With these assumptions and EDM stepsize, we achieve the first discretization complexity $K$ for the one-step RF-based models. Note that this complexity is the discretization complexity in the training process, and the complexity $N$ of multi-step models is the sample complexity in the sampling process (\cref{sec: multistep}).


\begin{restatable}{theorem}{thminstaflowresults}\label{thm:instaflow_results}
Assume \cref{ass: Manifold assumption}, \ref{ass:score approxiamation assumption}, \ref{ass:onestepfmatcherror}, \ref{ass:flip} holds and consider the EDM stepsize. Then, one-step generation error $W_2\left(f_{\theta, 0} \sharp \mathcal{N}\left(0,  I_d\right), q^*\right)$ is bounded by
\begin{align*}
\frac{L_fR^2(R \vee \sqrt{d})}{\delta^{2+1/a}K} + L_f\varepsilon_{sc}+ \varepsilon_{cm}+(\sqrt{d}+R)\delta\,.
\end{align*}
Furthermore, by choosing $\delta=\epsilon_{W_2}/(\sqrt{d}+R)$, $\epsilon_{cm}\leq \epsilon_{W_2}$ and  $\epsilon_{sc}\leq \epsilon_{W_2}/L_f$, the output is $\epsilon_{W_2}$-close to $q^*$ with discretization complexity
\begin{align*}
K=O\left(\frac{L_fR^2(\sqrt{d}+R)^{2+1/a}(R \vee \sqrt{d})}{\epsilon_{W_2}^{3+1/a}}\right)\,.
\end{align*}
\end{restatable}

As shown in \cref{tab:iteration results}, this result is better than the previous one-step VP-based and VE-based models and makes the first step to explain why one-step RF-based models have a great performance in application. For the VP-based models, \citet{lyu2023convergenceconsistencyv1} achieve $L_f/\epsilon_{W_2}^7$ result, which has a worse $W_2$ dependence compared with our results. \citet{li2024towardsconsistencyv2} and \citet{doutheoryconsistency} achieve the discretization complexity $L_f^3/\epsilon_{W_1}$ and $L_f^2/\epsilon_{W_1}^{10}$ with a weaker $W_1$ guarantee. Furthermore, \citet{li2024towardsconsistencyv2} have a worse $L_f$ dependence, and \citet{doutheoryconsistency} have a worse $L_f$ and $\epsilon_{W_1}$ dependence. As shown in Theorem C.1 of \citet{lyu2023convergenceconsistencyv1}, $L_f$ has an order of $e^{C_2R^2}$ for some $C_2\ge0$, which indicates $L_f$ is a large Lipschitz constant. For the VE-based models, due to an additional polynomial $T^{1/a}$ (the order of $T$ is $1/\epsilon_{W_2}$, discussed in \cref{subsec:technique novelty}), the $L_{f}/\epsilon_{W_2}^{3+2/a}$ result is worse than our results.

\section{Conclusion}\label{sec:conclusion}
In this work, we analyze the RF-based models with a deterministic sampler under the multi-step and one-step settings. We begin by thoroughly examining the properties of RF-based models, demonstrating that they exhibit improved dependence on the early stopping parameter $\delta$ and diffusion time $T$ compared with VP, VE and DFM-based models. Then, based on these observations, we achieve the first polynomial discretization complexity for the multi-step and one-step RF-based models simultaneously, which both enjoy a better dependence compared with the results of diffusion-based models. Thus, this work takes the first step towards explaining why RF-based models achieve state-of-the-art performance from the perspective of discretization complexity.
\paragraph{Future Work and Limitation.} In this work, we directly assume the vector field and the one-step mapping function are accurate enough and mainly focus on the sample (discretization) complexity for the RF-based models. It is an interesting future work to analyze the statistical complexity and the training dynamics of multi-step and one-step RF-based models, which explains the success of RF-based models from the learning perspective and leads to an end-to-end analysis.

\appendix
\newpage
\onecolumn
\section*{Appendix}
\section{Notation}\label{sec:app notation}
\paragraph{Common notations.}
\begin{itemize}
    \item $X_0=Z\sim \mathcal{N}(0,I_d)$: the standard Gaussian distribution. $X_1\sim q^*$: the target data.
    \item $X_t=(1-t)Z+tX_1, t\in [0,1]$: the random variable for the linear interpolation. 
    
    $(q_t)_{t\in [0,1]}$: the marginal distribution of $(X_t)_{t\in [0,1]}$.
    \item $v^{*}(t,X)$ is the ground truth velocity field.  $\hat{v}(t,X)$: the estimated velocity field.
    
\end{itemize}

\paragraph{Multi-step RF-based models with deterministic sampler and ULD corrector.}
We define by $h_{\mathrm{pred}}$ and $h_{\mathrm{corr}}$ the stepsize of predictor and corrector and $N'$ the total number of stages of \cref{algo:algo1}. Let $p_{1-\delta}^{\operatorname{ULMC}}$ be the output of \cref{algo:algo1}. The iteration complexity contains two parts: the predictor and the corrector iteration complexity. The goal is to guarantee $\operatorname{TV}(p_{1-\delta}^{\operatorname{ULMC}},q_{1-\delta})\leq \epsilon_{\operatorname{TV}}$ and $W_2(q_{1-\delta},q^*)\leq \epsilon_{W_2}$.
\begin{itemize}
    \item The predictor iteration complexity: $N'T_{\mathrm{pred}}/h_{\mathrm{pred}}= (1-\delta))/h_{\mathrm{pred}}$.
    \item The corrector iteration complexity: $\frac{N'T_{\mathrm{corr}}}{h_{\mathrm{corr}}}$ (Since the corrector runs $T_{\mathrm{corr}}$ time at each stage). 
    \item For distribution $q$, we denote $qP_{flow}^{t_0,h}$ as running the flow ode at time $t_0$ with step size $h$ using the ground truth velocity. 
    \begin{align*}
   \mathrm{d} X_t=v^*\left(t, X_t\right)\mathrm{~d} t, \quad X_{t_0}\sim q . \quad X_{t_0+h} \sim qP_{flow}^{t_0,h}
\end{align*}
    \item 
For distribution $q$, we denote $q\hat{P}_{flow}^{t_0,h}$ as running the discretized flow ode at time $t_0$ with stepsize $h$ using the approximated velocity field.
\begin{align*}
    \mathrm{d} \hat{X}_t=\hat{v}\left(t_{0}, \hat{X}_{t_0}\right)\mathrm{~d} t\,,\hat{X}_{t_0}\sim q. \quad \hat{X}_{t_0+h} \sim q\hat{P}_{flow}^{t_0,h}
\end{align*}
    \item $P_{\mathrm{ULD}}$ is the output of running the continuous underdamped Langevin corrector for time $h$.
    
    $\hat{P}_{\text{ULMC}}$ are the corresponding implementable algorithm with estimated score.
\end{itemize}

\paragraph{One-step RF-based models.}
\begin{itemize}
    \item The goal of one-step RF-based models is to learn a one-step mapping function $f_\theta\left(X, t\right)$ to directly map pure noise $X\sim \mathcal{N}(0,I_d)$ and $t=0$ to $q^*$.
    \item  $f_{\theta,0}$: the learned one-step mapping function at $t=0$.
    \item $0=t_0\leq t_1\leq \cdots\leq t_K=1-\delta$: the discretization points used in the training process of the one-step generative models. $h_k=t_k-t_{k-1}\equiv h,$ $\forall k\in \{1,\cdots,K\}$: the uniform stepsize. 
    \item We denote by $f_{\theta, 0} \sharp \mathcal{N}\left(0, I_d\right)$ the generated distribution of the above operation. Since the one-step mapping function is one step, the discretization complexity is the requirement of $K=(1-\delta)/h$ in the training process (\cref{eq:CD objective}) to guarantee $W_2(f_{\theta, 0} \sharp \mathcal{N}\left(0, I_d\right),q^*)\leq \epsilon_{W_2}$.
\end{itemize}

\section{The Detailed Calculation of Previous work}\label{subsec: detailed cal previous work}

\paragraph{The early stopping parameter $\delta$ of \citet{silveri2024theoreticalflowkl}.}
As shown in Eq. (5) of \citet{silveri2024theoreticalflowkl}, the stochastic linear interpolation has the following form:
\begin{align*}
    X_t = (1-t)Z+tX_1+\sqrt{2t(1-t)}Z\,,t\in [0,1]\,,
\end{align*}
which indicates 
\begin{align*}
W_2^2\left(q_{1-\delta}, q_1\right) & \leq \mathbb{E}\left[\left\|X_{1-\delta}-X_1\right\|_2^2\right]=\mathbb{E}\left[\left\|\delta\left(\mathrm{Z}-X_1\right)+\sqrt{2\delta(1-\delta)Z}\right\|_2^2\right] \\
& =\delta^2\left(\mathbb{E}\left[\|Z\|_2^2\right]+\mathbb{E}\left[\left\|X_1\right\|_2^2\right]\right)+\delta(\mathbb{E}\left[\|Z\|_2^2\right]\leq (R^2+d)\delta^2+d\delta\,.
\end{align*}
To make the above inequality smaller than $\epsilon_{W_2}^2$, the $\delta$ has the order of $\epsilon_{W_2}^2$.

\paragraph{The results of \citet{doutheoryconsistency}.} 

As shown in Theorem 4.1 of \citet{doutheoryconsistency}, when considering $W_1$ distance, the total error has the following form (Here, we only discuss the dependence on $\epsilon_{W_1}$ and ignore $R, d$.) 
\begin{align*}
    \frac{L_fL_{S,X}}{\sqrt{M}}+\sqrt{\delta}\,,
\end{align*}
where $M$ is the number of the discretization steps. To make the above term smaller than $\epsilon_{W_1}$, we require 
\begin{align*}
    M\ge  \frac{L_f^2L_{S,X}^2}{\epsilon_{W_1}^2} \text{ and } \delta\leq \epsilon_{W_1}^2\,.
\end{align*}
As shown in \citet{chen2022sampling}, we know that $\sigma_{\delta}^2=\delta$ for VPSDE forward process, and we need to choose $L_{\text{score}}=R^2/\sigma_{\delta}^4=R^2/\delta^2$, which finally requires 
\begin{align*}
    M\ge \frac{L_f^2}{\epsilon_{W_1}^{10}}\,.
\end{align*}

\section{Property of Optimal Vector Field and Vector Perturbation Lemma}\label{subsec:the property of optimal vector}

In this section, we first introduce auxiliary lemmas on the vector field. Then, we provide the proof for our vector perturbation lemma. Similar with \cite{gao2024convergence}, we define some notation, which is  useful in this section.
$$M_1 := \mathbb{E}[X_1|X_t = X], 
M_2 := \mathbb{E}[X_1^T X_1|X_t = X], 
M_2^c := \operatorname{Cov}(X_1|X_t = X)\,,  
$$
and $M_3 := \mathbb{E}[X_1 X_1^T X_1|X_t = X]$.

\begin{lemma}\label{lem9}
$v^*(t,X)$ can be expressed by moment as follows:
$$v^*(t,X)=-\frac1{1-t}X+\frac1{1-t}M_1,\quad(t,X)\in[0,1)\times\mathbb{R}^d.$$  
\end{lemma}
\begin{proof}
    Using the fact that $X = \mathbb{E} \left[ (1 - t)Z + tX_1 \mid X_t = X \right]$, we can do an easy transformation of $v^*(t,X)=\mathbb{E}\left[X_1-Z |X_t=X\right]$ to finish the proof.
\end{proof}

In the following two lemmas, we provide the order of Lipschitz constant $L_{v,X}$ and $L_{v,t}$. 

\begin{lemma}\label{lem10}
The $\nabla_Xv^*(t,X)$ term can be expressed in covariance as follows:$$\nabla_Xv^*(t,X)=\frac t{(1-t)^3}M_{2}^{c}-\frac1{1-t}\mathrm{I}_d,\quad(t,X)\in[0,1)\times\mathbb{R}^d.$$
\end{lemma}

\begin{lemma}\label{lem11}
The time derivative of the
velocity field $v^{*}$ can be expressed by moment as follows:
$\partial_tv^*(t,X)=-\frac{1}{(1-t)^2}X+\frac{1}{(1-t)^2}M_1+\frac{t+1}{(1-t)^4}M_2^cX-\frac{t}{(1-t)^4}(M_3-M_2M_1)$.
\end{lemma}

We also provide the order of the early stopping parameters $\delta$.
\begin{lemma}\label{lem:orderofdelta}
Assume \cref{ass: Manifold assumption} holds. Then, we have the following bound:
\begin{align*}
    W_2^2(q_{1-\delta},q^*)\leq (d+R^2)\delta^2\,.
\end{align*}
\end{lemma}
Lemma \ref{lem10}, Lemma \ref{lem11} and Lemma \ref{lem:orderofdelta} comes from Lemma 3.3, Lemma 3.4 and Lemma 4.3 of \cite{gao2024convergence}.

\subsection{Proof for the Vector Perturbation Lemma}
\label{v-perturbationlemma}
\vperturbation*

\begin{proof}
\quad First, we have $$\partial_t v^*(t,X_t) = \partial_t v^*(t,X)|_{X=X_t} + \nabla_Xv^*(t,X_t)v^*(t,X_t)\,,$$
which indicates 
$$\mathbb{E}\|\partial_t v^*(t,X_t)\|^2 \lesssim \mathbb{E}\|\partial_t v^*(t,X)|_{X=X_t}\|^2 + \mathbb{E}\|\nabla_Xv^*(t,X_t)v^*(t,X_t)\|^2$$
    
    By the relation of $\partial_t v^*(t,X)$ with the conditional moments (Lemma \ref{lem11}), we have that
    \begin{align*}
        \partial_tv^*(t,X)&=-\frac{1}{(1-t)^2}X+\frac{1}{(1-t)^2}M_1+\frac{t+1}{(1-t)^4}M_2^cX-\frac{t}{(1-t)^4}(M_3-M_2M_1)\\&\lesssim \frac{\left(R^2\|X\|_2\right) \vee R^3}{(1-t)^4}\,.
    \end{align*}
We also know that 
\begin{align*}
    \underset{X \sim q_t}{\mathbb{E}}\|X\|^2=\mathbb{E}\|(1-t)X_0+tX_1\|^2\lesssim(1-t)^2\mathbb{E}\|X_0\|^2\lesssim(1-t)^2d+t^2R^2\lesssim R^2 \vee d\,.
\end{align*}
Then, we can bound the first term by $$\mathbb{E}\|\partial_t v^*(t,X)|_{X=X_t}\|^2 \lesssim \frac{\underset{X\sim q_t}{\mathbb{E}}\|
    \left(R^2\|X\|_2\right) \vee R^3\|^2}{(1-t)^8}=\frac{R^6 \vee (R^4\underset{X \sim q_t}{\mathbb{E}}\|X\|^2)}{(1-t)^8}\lesssim \frac{R^4(R^2 \vee d)}{(1-t)^8}$$

    By the relation of $\nabla_X v^*$ and $v^*$ with the conditional moments we have,$$\|\nabla_Xv^*(t,X)\|_{2,2}=\|\frac t{(1-t)^3}M_{2}^{c}-\frac1{1-t}\mathrm{I}_d\| \lesssim \frac{R^2}{(1-t)^3}\,.$$
As a result,s
we can bound the second term by $$\mathbb{E}\|\nabla_Xv^*(t,X_t)v^*(t,X_t)\|^2 \lesssim \frac{R^4}{(1-t)^6} \mathbb{E}\|v^*(t,X_t)\|^2 \lesssim\frac{R^4}{(1-t)^8}\mathbb{E}\|X\|_2^2 \vee R^2 \lesssim \frac{R^4(R^2 \vee d)}{(1-t)^8}\,.$$

Combining the above inequalities, we obtain $$\mathbb{E}\|\partial_t v^*(t,X_t)\|^2 \lesssim \frac{R^4(R^2 \vee d)}{(1-t)^8} + \frac{R^4(R^2 \vee d)}{(1-t)^8} \lesssim \frac{R^4(R^2 \vee d)}{(1-t)^8}$$

 If choosing union bound for early stopping parameter $\delta$, then for $t \in [0, 1-\delta]$, we know that  $\mathbb{E}\|\partial_t v^*(t,X_t)\|^2 \lesssim \frac{R^4(R^2 \vee d)}{\delta^8}$.
\end{proof}

\section{Analysis for Algorithm 1}
In this section, we provide the proof of \cref{thm:main_theorem_reverse_PFODE}. We first provide the bound for the predictor stage (\cref{sec:prediction step}) and corrector stage (\cref{sec:the proof for corrector stage}), respectively. Then, we combine these two parts and finish the proof.

\subsection{The Proof of Predictor Stage}
\label{sec:prediction step}
\begin{lemma}
\label{prediction step}
Assuming \cref{ass: Manifold assumption}, \ref{ass:score approxiamation assumption}, \ref{hatvlip} holds and $h \lesssim 1/L_{v,X}$. Then  for any distribution $q$, we have the following bound:
\begin{align*}
    W_2(qP_{flow}^{t_0,h}, q\hat{P}_{flow}^{t_0,h}) \lesssim h^2\frac{R^2\sqrt{R^2 \vee d}}{\delta^4} + h\varepsilon_{sc}.
\end{align*}

\end{lemma}

\begin{proof}
Suppose $X_{t_0} = \hat{X}_{t_0}\sim q $. By the definition of the predictor step, we have that  
\begin{align*}
\frac{\mathrm{d}X}{\mathrm{d}t} &= v^*(t,X_t) \\
\frac{\mathrm{d}\hat{X}}{\mathrm{d}t} &= \hat{v}(t,\hat{X}_t)\,.
\end{align*}

As a result, we can derive the difference by
\begin{align*}
\partial_t \|X_t - \hat{X}_t\|^2 &= 2 \langle X_t - \hat{X}_t, \frac{\mathrm{d}X}{\mathrm{d}t} -\frac{\mathrm{d}\hat{X}}{\mathrm{d}t} \rangle \\
&=  \langle X_t - \hat{X}_t, v^*(t,X_t)-\hat{v}(t,\hat{X}_t)\rangle \\
&\leq \frac{1}{h} \|X_t - \hat{X}_t\|^2 + h \|v^*(t,X_t)-\hat{v}(t,\hat{X}_t)\|^2.
\end{align*}

By Gr\"{o}nwall's inequality, we have that 

\begin{equation*}
\mathbb{E}[\|X_{t_0+h} - \hat{X}_{t_0+h}\|^2] \leq \exp( \frac{1}{h}\cdot h) \int_{t_0}^{t_0+h} h \mathbb{E}[\|v^*(t,X_t)-\hat{v}(t,\hat{X}_t)\|^2]  dt
\end{equation*}

\begin{equation}
\lesssim h \int_{t_0}^{t_0+h} \mathbb{E}[\|v^*(t,X_t)-\hat{v}(t,\hat{X}_t)\|^2]dt. \label{equ5}
\end{equation}

We can decopose the error in the way that

\begin{align*}
\mathbb{E}[\|v^*(t,X_t) - \hat{v}(t,\hat{X}_t)\|^2] &= \mathbb{E}[\|v^*(t,X_t) \underbrace{ - v^*(t_0,X_{t_0}) + v^*(t_0,\hat{X}_{t_0}) }_{=0} + \hat{v}(t,\hat{X}_t)\|^2] \\
&\lesssim \mathbb{E}[\|v^*(t,X_t) - v^*(t_0,X_{t_0})\|^2] + \mathbb{E}[\|v^*(t_0,\hat{X}_{t_0}) - \hat{v}(t_0,\hat{X}_{t_0})\|^2].
\end{align*}
Note that $X_{t_0} \sim \hat{X}_{t_0}$, $X_{t_0}$ and $\hat{X}_{t_0}$ has the same distribution.

By lemma \ref{vperturbation}, for $t \in [t_0,t_0+h]$ the first term is bounded by 
\begin{align*}
\mathbb{E}[\|v^*(t,X_t) - v^*(t_0,X_{t_0})\|^2]&= \mathbb{E}[\|\int_{t_0}^t \partial_sv^*(s,X_s)ds\|^2]\\
&\leq (t-t_0) \int_{t_0}^t \mathbb{E}[\|\partial_sv^*(s,X_s)\|^2]ds\\
&\lesssim(t-t_0)^2\frac{R^4(R^2 \vee d)}{\delta^8}\\
&\lesssim h^2 \frac{R^4(R^2 \vee d)}{\delta^8}
\end{align*}

By \cref{ass:score approxiamation assumption}, we know that the second term is bounded 
\begin{equation*}
\mathbb{E}[\|v^*(t_0,\hat{X}_{t_0}) - \hat{v}(t_0,\hat{X}_{t_0})\|^2]  \leq \varepsilon^2_{sc}  
\end{equation*}

Combining above results and using \cref{equ5}, we have that
\begin{equation}\label{running one step}
    \mathbb{E}[\|X_{t_0+h} - \hat{X}_{t_0+h}\|^2] \lesssim  h^4 \frac{R^4(R^2 \vee d)}{\delta^8} + h^2 \varepsilon^2_{sc}\,.
\end{equation}

So we have $$W_2(q\hat{P}_{flow}^{t_0,h},qP_{flow}^{t_0,h})\leq \sqrt{\mathbb{E}[\|X_{t_0+h} - \hat{X}_{t_0+h}\|^2]} \lesssim  h^2 \frac{R^2(R \vee \sqrt{d})}{\delta^4} + h \varepsilon_{sc}\,.$$
\end{proof}

\begin{lemma}
\label{lemma8}
Assuming \cref{ass: Manifold assumption}, \ref{ass:score approxiamation assumption}, \ref{hatvlip} holds and \(h \lesssim 1/L_{v,X}\). Then  for any two distribution $p_1$ and $p_2$, after both running a h time step using discretized velocity at any time $t_0$, we have the following guarantee:
$$W_2^2(p_1P_{flow}^{t_0,h},p_2P_{flow}^{t_0,h}) \leq e^{O(L_{v,X}h)}W_2(p_1,p_2)$$
\end{lemma}

\begin{proof}
Suppose $X_{t_0} \sim p_1,\hat{X}_{t_0}\sim p_2 $, and the couple $\gamma \sim (p_1,p_2)$ makes $W_2(p_1,p_2)=\underset{(x_{t_0},\hat{x}_{t_0})\sim \gamma}{\mathbb{E}}[\|x_{t_0}-\hat{x}_{t_0}\|^2]^{\frac{1}{2}}$. Then, we have that
\begin{align*}
\frac{\mathrm{d}X}{\mathrm{d}t} &= v^*(t_0,X_{t_0}) \\
\frac{\mathrm{d}\hat{X}}{\mathrm{d}t} &= v^*(t_0,\hat{X}_{t_0})\,,
\end{align*}
which indicates
\begin{align*}
\partial_t \|X_t - \hat{X}_t\|^2 &= 2 \langle X_t - \hat{X}_t, \dot{X}_t - \dot{\hat{X}}_t \rangle \\
&=   2\langle X_t - \hat{X}_t, v^*(t_0,x_{t_0})-v^*(t_0,\hat{X}_{t_0})\rangle \\
&\leq L_{v,X} \|X_t - \hat{X}_t\|^2 + \frac{1}{L_{v,X}}\|v^*(t_0,X_{t_0})-v^*(t_0,\hat{X}_{t_0})\|^2\\
&\leq L_{v,X} \|X_t - \hat{X}_t\|^2 +\frac{1}{L_{v,X}}\cdot L_{v,X}^2\|X_{t_0}-\hat{X}_{t_0}\|^2\,.
\end{align*}

By Gr\"{o}nwall's inequality,

\begin{align*}
W_2^2(p_1P_{flow}^{t_0,h},p_2P_{flow}^{t_0,h})&\leq \mathbb{E}[\|X_{t_0+h}-\hat{X}_{t_0+h}\|^2]\\
&\lesssim \mathbb{E}[e^{L_{v,X}h}(\|X_{t_0}-\hat{X}_{t_0}\|^2+L_{v,X}h\|X_{t_0}-\hat{X}_{t_0}\|^2)]\\
&\lesssim e^{L_{v,X}h}\mathbb{E}[\|X_{t_0}-\hat{X}_{t_0}\|^2]\\
&=e^{L_{v,X}h}W_2^2(p_1,p_2)
\end{align*} 

So we have $$W_2^2(p_1P_{flow}^{t_0,h},p_2P_{flow}^{t_0,h}) \leq \exp(O(L_{v,X}h))W_2(p_1,p_2)$$
\end{proof}

\begin{lemma}
\label{lem:predict-n-step}
Assuming \cref{ass: Manifold assumption}, \ref{ass:score approxiamation assumption}, \ref{hatvlip} holds. Suppose $h_1,h_2,\cdots,h_{N_0}>0$, $t_N = h_1 +h_2+\cdots+h_{N_0}\lesssim L_{v,X}$, denote $h_{max} = \underset{i \in N_0}{\max{h_i}}$, we have $W_2(qP_{flow}^{t_0,h_1,\cdots,h_{N_0}},q\hat{P}_{flow}^{t_0,h_1,\cdots,h_{N_0}})\lesssim \sum_{i=1}^{N_0}h_i^2 \frac{R^2(R \vee \sqrt{d})}{\delta^4} + h_i \varepsilon_{sc}$
\end{lemma}

\begin{proof}
\begin{align*} &W_2(qP_{flow}^{t_0,h_1,\cdots,h_{N_0}},q\hat{P}_{flow}^{t_0,h_1,\cdots,h_{N_0}})\\&
\leq W_2(qP_{flow}^{t_0,h_1,\cdots,h_N},qP_{flow}^{t_0,h_1,\cdots,h_{N_0-1}}\hat{P}_{flow}^{h_{N_0}}) + W_2(qP_{flow}^{t_0,h_1,\cdots,h_{N_0-1}}\hat{P}_{flow}^{h_{N_0}},q\hat{P}_{flow}^{t_0,h_1,\cdots,h_N})\\
&\leq O(h_{N_0}^2 \frac{R^2(R \vee \sqrt{d})}{\delta^4} + h_{N_0} \varepsilon_{sc}) + \exp(O(L_{v,X}h_{N_0}))W_2(qP_{flow}^{t_0,h_1,\cdots,h_{N_0-1}},q\hat{P}_{flow}^{t_0,h_1,\cdots,h_{N_0-1}})\,,
\end{align*} 
where the first term is by lemma \ref{prediction step} and the second term is by lemma \ref{lemma8}. Then by recursion, we have 
\begin{align*}
&W_2(qP_{flow}^{t_0,h_1,\cdots,h_{N_0}},q\hat{P}_{flow}^{t_0,h_1,\cdots,h_{N_0}})\\&\leq ((\sum_{i=1}^{N_0}h_i^2 \frac{R^2(R \vee \sqrt{d})}{\delta^4} + h_i \varepsilon_{sc}))\cdot \exp(O(L_{v,X}(h_1+h_2+\cdots+h_{N_0})))\\
&\lesssim \sum_{i=1}^{N_0}h_i^2 \frac{R^2(R \vee \sqrt{d})}{\delta^4} + h_i \varepsilon_{sc}\,.
\end{align*}
\end{proof}

\subsection{The Proof for Correct Stage}\label{sec:the proof for corrector stage}
In this section, we use the underdamped Langevin corrector to inject noise to convert the $W_2$ distance to the TV distance so that we can guarantee that the data processing inequality is preserved. This part is a comprehensive explanation of \cref{subsec:introduce of ULC} and uses the same notation.

We denote the prediction time as $T_{\mathrm{pred}}$ and $N_{\mathrm{pred}}$ as the whole steps during $T_{\mathrm{pred}}$. Denote $q_{t_0}=law(X_{t_0}).\quad X_{t_0}=(1-t_0)Z+t_0X_1$. From \cref{sec:prediction step}, starting at the same distribution $q_{t_0}$, we run the continuous flow process and discretized process for $T_{\mathrm{pred}}$, we get respectively $q:=q_{t_0}P_{flow}^{t_0,N_{\mathrm{pred}}},p:=q_{t_0}\hat{P}_{flow}^{t_0,N_{\mathrm{pred}}}$.

To align with the notation \cref{subsec:introduce of ULC}, we define by probability measure $\boldsymbol{q}:=q \otimes \gamma^d$,and $\boldsymbol{p}:=p \otimes \gamma^d$  where $\gamma^d=\mathcal{N}(0,I)$. Throughout this part, we set the friction parameter of ULD to $\rho  \asymp \sqrt{L_{S,X}}$ with $L_{S,X}=R^2/\delta^4$.

We want to bound the TV distance of  $\boldsymbol{p}P_{\mathrm{ULMC}}^{t_0+t_{\mathrm{pred},N_{\mathrm{corr}}}}$ and $\boldsymbol{q}$, where $P_{\mathrm{ULMC}}^{t_0+t_{\mathrm{pred},N_{\mathrm{corr}}}}$ is the transition kernel of running \cref{eq:ULMC}. Using triangle inequality for TV distance, we can decompose $\operatorname{TV}(\boldsymbol{p}P_{\mathrm{ULMC}}^{t_0+t_{\mathrm{pred},N_{\mathrm{corr}}}},\boldsymbol{q})$ into $\operatorname{TV}(\boldsymbol{p}P_{\mathrm{ULMC}}^{t_0+t_{\mathrm{pred},N_{\mathrm{corr}}}},\boldsymbol{p}P_{\mathrm{ULD}}^{t_0+t_{\mathrm{pred},N_{\mathrm{corr}}}})$ and $\operatorname{TV}(\boldsymbol{p}P_{\mathrm{ULD}}^{t_0+t_{\mathrm{pred},N_{\mathrm{corr}}}},\boldsymbol{q})$, where $P_{\mathrm{ULD}}^{t_0+t_{\mathrm{pred},N_{\mathrm{corr}}}}$ is the transition kernel of running \cref{eq:ULD}

The following lemma comes from Lemma 9 of \cite{chen2023probability} and can be used to bound $\operatorname{TV}(\boldsymbol{p}P_{ULD}^{t_0+t_{\mathrm{pred},N_{\mathrm{corr}}}},\boldsymbol{q})$.

\begin{lemma}
\label{lemma10}
Assuming \cref{ass: Manifold assumption}, \ref{ass:score approxiamation assumption}, \ref{hatvlip} holds. If $T_{\mathrm{corr}} \lesssim 1/\sqrt{L_{S,X}}$, then $$\operatorname{TV} \left( pP_{ULD}^{t_0+T_{\mathrm{pred}},N_{\mathrm{corr}}}, q \right) \lesssim \sqrt{\operatorname{KL} \left( pP_{ULD}^{t_0+T_{\mathrm{pred}},N_{\mathrm{corr}}} \| q \right)} \lesssim \frac{W_2(p,q)}{L_{S,X}^{1/4} T_{\mathrm{corr}}^{3/2}}.$$
\end{lemma}

Next, we need to bound $\operatorname{TV}(\boldsymbol{p}P_{ULMC}^{t_0+t_{\mathrm{pred},N_{\mathrm{corr}}}},\boldsymbol{p}P_{ULD}^{t_0+t_{\mathrm{pred},N_{\mathrm{corr}}}})$, which corresponds just to the discretized error. Recall we just use the score function at $t_0+T_{\mathrm{pred}}$ here, we introduce two process.   
\begin{align*}
\mathrm{d}z_m^\circ &= v_m^\circ \, \mathrm{d}m, \\
\mathrm{d}v_m^\circ &= -\rho v_m^\circ \, \mathrm{d}m - \nabla U_{t_0+T_{\text{pred}}} (z_m^\circ) \, \mathrm{d}m + \sqrt{2} \rho \, \mathrm{d}B_m, \quad (z_0^\circ, v_0^\circ) \sim \boldsymbol{q}, \\
\mathrm{d}z_m &= v_m \, \mathrm{d}m, \\
\mathrm{d}v_m &= -\rho v_m \, \mathrm{d}m - \nabla U_{t_0+T_{\mathrm{pred}}} (z_m) \, \mathrm{d}m + \sqrt{2} \rho \, \mathrm{d}B_m, \quad (z_0, v_0) \sim \boldsymbol{p}.
\end{align*}

Since $\boldsymbol{q}$ is already stationaty distribution, we know for ant integer k, $$\left( z^\circ_{nh_{\mathrm{corr}}}, v^\circ_{nh_{\mathrm{corr}}} \right) 
\sim \boldsymbol{q} P^{k}_{ULD} = \boldsymbol{q}, \quad 
\left( z_{kh_{\mathrm{corr}}}, v_{kh_{\mathrm{corr}}} \right) 
\sim \boldsymbol{p} P^{k}_{ULD}.$$

We use the following lemma form Lemma10 of \cite{chen2023probability} to bound the $W_2$ distance between the two process.

\begin{lemma}
\label{lemma11}
 Assuming \cref{ass: Manifold assumption}, \ref{ass:score approxiamation assumption}, \ref{hatvlip} holds.   If $T_{\mathrm{corr}}\lesssim1/\sqrt{L_{S,X}}$, then for $m\leq T_{\mathrm{corr}}$, we have $$\mathbb{E}[\|z_m^\circ-z_m\|^2]\leq W_2^2(p,q)$$.
\end{lemma}

Finally, we can we use the folloing lemma to bound the discretization error.

\begin{lemma}
\label{lem:ULMCdiserro}
Assuming \cref{ass: Manifold assumption}, \ref{ass:score approxiamation assumption}, \ref{hatvlip} holds. If $T_{\mathrm{corr}}\leq1/\sqrt{L_{S,X}}$, then we have that 
\begin{equation*}
\operatorname{TV} \left( p\hat{P}_{\mathrm{ULMC}}^{t_0+T_{\mathrm{pred}},N_{\mathrm{corr}}}, pP_{\mathrm{ULD}}^{t_0+T_{\mathrm{pred}},N_{\mathrm{corr}}} \right) 
\lesssim \sqrt{\operatorname{KL} \left( pP_{\mathrm{ULMC}}^{t_0+T_{\mathrm{pred}},N_{\mathrm{corr}}} \| p\hat{P}_{\mathrm{ULD}}^{t_0+T_{\mathrm{pred}},N_{\mathrm{corr}}}  \right)} 
\end{equation*}
\begin{equation*}
\lesssim L_{S,X}^{3/4}T_{\mathrm{corr}}^{1/2}W_2(p,q)+   L_{S,X}^{3/4}T_{\mathrm{corr}}^{1/2}\sqrt{d} h_{\mathrm{corr}} + \frac{T_{\mathrm{corr}}^{1/2}\varepsilon_{sc}}{L_{S,X}^{1/4}\delta} .
\end{equation*}
\end{lemma}

\begin{proof}
The first part inequality is by Pinsker’s inequality.

Then we use similar argument in \citet{chen2023probability}. Using similar approximate argument and Girsanov’s theorem,we have$$\operatorname{KL}\left(\boldsymbol{p}P_{\mathrm{ULD}}^{t_0+T_{\mathrm{pred}},N_{\mathrm{corr}}} \middle\| \boldsymbol{p}\hat{P}_{\mathrm{ULMC}}^{t_0+T_{\mathrm{pred}},N_{\mathrm{corr}}}\right) \lesssim \frac{1}{\rho}\sum_{n=1}^{N_{\mathrm{corr}}} \int_{(n-1)h_{\mathrm{corr}}}^{nh_{\mathrm{corr}}} \mathbb{E}\left[\left\|s_{t_0+T_{\mathrm{pred}}}(z_{nh_{\mathrm{corr}}}) - \nabla U_{t_0+T_{\mathrm{pred}}}(z_u)\right\|^2\right] \mathrm{d}u.$$

Combined the relation between score and velocity (\cref{score velocity transform} and \cref{eq:transfer between score and vector}) with the velocity matching error (\cref{ass:score approxiamation assumption}), we have the approximation score error at ant time $t$.

And by \cref{hatvlip}, the approximate $\hat{v}(t,X)$is $L_{v,X}$-Lipschitz, the approximate score is also $L_{S,X}$-Lipschitz.

\begin{align*}
\mathbb{E}_{X_t\sim q_t}[\|\nabla \log q_t(X_t)-s(t,X_t)\|]=\frac{t^2}{(1-t)^4}\varepsilon_{sc}^2\leq \frac{\varepsilon_{sc}^2}{\delta^4}
\end{align*}

For the remaining part, we omit the time index $t_0+T_{\mathrm{pred}}$ without causing any ambiguity.

\begin{align*}
&\mathbb{E} \left[ \left\| s \left( z_{nh_{\mathrm{corr}}} \right) - \nabla U \left( z_u \right) \right\|^2 \right]\\
&\lesssim \mathbb{E} \left[ \left\| s \left( z_{nh_{\mathrm{corr}}} \right) - s \left( z^\circ_{nh_{\mathrm{corr}}} \right) \right\|^2 \right] 
+ \mathbb{E} \left[ \left\| s \left( z^\circ_{nh_{\mathrm{corr}}} \right) - \nabla U \left( z^\circ_{nh_{\mathrm{corr}}} \right) \right\|^2 \right]  \\
&\quad + \mathbb{E} \left[ \left\| \nabla U \left( z^\circ_{nh_{\mathrm{corr}}} \right) - \nabla U \left( z_u^\circ \right) \right\|^2 \right] + \mathbb{E}[\|\nabla U \left(z_u^\circ\right)-\nabla U \left(z_u\right)\|^2]\\
&\lesssim L_{S,X}^2 \, \mathbb{E} \left[ \left\| z_{nh_{\mathrm{corr}}} - z^\circ_{nh_{\mathrm{corr}}} \right\|^2 \right]+\frac{\epsilon_{score}^2}{\delta^2} \\ 
&\quad+ L_{S,X}^2 \mathbb{E} \left[ \left\| z^\circ_{nh_{\mathrm{corr}}} - z_u^\circ \right\|^2 \right] + L_{S,X}^2\mathbb{E}[\|z^\circ_u - z_u\|]\\
&\lesssim L_{S,X}^2 W_2^2 (p, q) 
+ L_{S,X}^2 \, \mathbb{E} \left[ \left\| z^\circ_{nh_{\mathrm{corr}}} - z^\circ_u \right\|^2 \right] 
+ \frac{\epsilon_{score}^2}{\delta^2}\,,
\end{align*}
where the second inequality is by Lemma \ref{ass:score approxiamation assumption} and  the last inequality is by Lemma \ref{lemma11}.
\begin{align*}
\mathbb{E}\left[\left\|z_{n h_{\mathrm{corr}}}^{\circ}-z_u^{\circ}\right\|^2\right]=\mathbb{E}\left[\left\|\int_{n h_{\mathrm{corr}}}^u v_s^{\circ} \mathrm{d} s\right\|^2\right] \leq h_{\mathrm{corr}} \int_{n h_{\mathrm{corr}}}^u \mathbb{E}\left[\left\|v_s^{\circ}\right\|^2\right] \mathrm{d} s \leq d h_{\mathrm{corr}}^2\,.
\end{align*}
The second inequality by the fact that $v_s^{\circ}$is always the stationaty distribution which is $v_s^{\circ} \sim \gamma^d$

Recall $\rho \asymp \sqrt{L_{S,X}}$, combing the  above result, we know that:
\begin{align*}
\operatorname{TV} \left( p\hat{P}_{ULMC}^{t_0+T_{\mathrm{pred}},N_{\mathrm{corr}}}, pP_{ULD}^{t_0+T_{\mathrm{pred}},N_{\mathrm{corr}}} \right) 
&\lesssim \sqrt{\operatorname{KL} \left( pP_{ULMC}^{t_0+T_{\mathrm{pred}},N_{\mathrm{corr}}} \| p\hat{P}_{ULD}^{t_0+T_{\mathrm{pred}},N_{\mathrm{corr}}}  \right)} 
\\&\lesssim L_{S,X}^{3/4}T_{\mathrm{corr}}^{1/2}W_2(p,q)+   L_{S,X}^{3/4}T_{\mathrm{corr}}^{1/2}\sqrt{d} h_{\mathrm{corr}} + \frac{T_{\mathrm{corr}}^{1/2}\varepsilon_{sc}}{L_{S,X}^{1/4}\delta} .
\end{align*}

\end{proof}

\begin{lemma}\label{lem:finalcorrector} Assuming \cref{ass: Manifold assumption}, \ref{ass:score approxiamation assumption}, \ref{hatvlip} holds. Define by $q:=q_{t_0}P_{flow}^{t_0,N_{\mathrm{pred}}}$ and $ p:=q_{t_0}\hat{P}_{flow}^{t_0,N_{\mathrm{pred}}}$. If choosing $T_{\mathrm{corr}}\asymp1/\sqrt{L_{S,X}}$, then we know that 

$$W_2(\boldsymbol{p}P_{ULMC}^{t_0+t_{\mathrm{pred},N_{\mathrm{corr}}}},\boldsymbol{q})\lesssim\sqrt{L_{S,X}}W_2(q_{t_0} \hat{P}_{\mathrm{flow}}^{t_0,N_{\mathrm{pred }}},q_{t_0+T_{\mathrm {pred }}})+   \sqrt{L_{S,X}d} h_{\mathrm{corr}} + \frac{\varepsilon_{sc}}{\sqrt{L_{S,X}}\delta}\,. $$
\end{lemma}
\begin{proof}
Combining Lemma \ref{lemma10} and Lemma \ref{lem:ULMCdiserro} and choosing $T_{\mathrm{corr}}\asymp1/\sqrt{L_{S,X}}$, we have the following bound:
\begin{align*}    W_2(\boldsymbol{p}P_{ULMC}^{t_0+t_{\mathrm{pred},N_{\mathrm{corr}}}},\boldsymbol{q})&\leq \operatorname{TV}(\boldsymbol{p}P_{ULMC}^{t_0+t_{\mathrm{pred},N_{\mathrm{corr}}}},\boldsymbol{p}P_{ULD}^{t_0+t_{\mathrm{pred},N_{\mathrm{corr}}}}) + \operatorname{TV}(\boldsymbol{p}P_{ULD}^{t_0+t_{\mathrm{pred},N_{\mathrm{corr}}}},\boldsymbol{q})\\&\lesssim
L_{S,X}^{3/4}T_{\mathrm{corr}}^{1/2}W_2(p,q)+   L_{S,X}^{3/4}T_{\mathrm{corr}}^{1/2}\sqrt{d} h_{\mathrm{corr}} + \frac{T_{\mathrm{corr}}^{1/2}\varepsilon_{sc}}{L_{S,X}^{1/4}\delta}+ \frac{W_2(p,q)}{L_{S,X}^{1/4} T_{\mathrm{corr}}^{3/2}}\\& \lesssim\sqrt{L_{S,X}}W_2(q_{t_0} \hat{P}_{\mathrm{flow}}^{t_0,N_{\text {pred }}},q_{t_0+T_{\text {pred }}})+   \sqrt{L_{S,X}d} h_{\mathrm{corr}} + \frac{\varepsilon_{sc}}{\sqrt{L_{S,X}}\delta}\,.
\end{align*}    
\end{proof}
\subsection{Proof for Multi-step RF-based Models}
Starting from a distribution $p$, if the algorithm runs a predictor stage followed by a correct stage, we have the following lemma.

\begin{lemma}
\label{lem:finallemma}
Assuming \cref{ass: Manifold assumption}, \ref{ass:score approxiamation assumption}, \ref{hatvlip} holds. Let $T_{\mathrm{pred}}=N_{\mathrm{pred}}h_{\mathrm{pred}}\lesssim 1/L_{v,X}$, $T_{\mathrm{corr}}=N_{\mathrm{corr}}h_{\mathrm{corr}}\lesssim1/\sqrt{L_{S,X}}$ and choose prediction step size $h_{\mathrm{pred}}$ and correction step size $h_{\mathrm{corr}}$ uniformly. Denote $q_{t}$ the stationary distribution at time $t$ and p the middle distribution of our algorithm. Then, we have that
\begin{align*}
&\operatorname{TV}(p\hat{P}_{flow}^{t_0,N_{\mathrm{pred}}}\hat{P}_{\mathrm{ULMC}}^{t_0+T_{\mathrm{pred}},N_{\mathrm{corr}}},q_{t_0+T_{\mathrm{pred}}})\\ &\lesssim\operatorname{TV}(p,q_0) + \sqrt{L_{S,X}}\frac{R^2(R \vee \sqrt{d})}{L_{v,X}\delta^4}h_{\text{pred}}+   \sqrt{L_{S,X}d} h_{\mathrm{corr}} + (\frac{1}{\sqrt{L_{S,X}}\delta}+\frac{\sqrt{L_{S,X}}}{L_{v,X}})\varepsilon_{sc}\,.
\end{align*}
\end{lemma}

\begin{proof}
By triangle inequality and data processing inequality, we have that
\begin{align*}
&\operatorname{TV}\left(p \hat{P}_{\mathrm{flow}}^{t_0,N_{\text {pred }}} \hat{P}_{\mathrm{ULMC}}^{t_0+T_{\text{pred}},N_{\text{corr }}}, q_{t_0+T_{\text {pred }}}\right)\\
& \leq \mathrm{TV}\left(p \hat{P}_{\mathrm{flow}}^{t_0,N_{\text {pred }}}\hat{P}_{\mathrm{ULMC}}^{t_0+T_{\text{pred}},N_{\text {corr }}}, q_{t_0} \hat{P}_{\mathrm{flow}}^{t_0,N_{\text {pred }}}\hat{P}_{\mathrm{ULMC}}^{t_0+T_{\text{pred}},N_{\text {corr }}}\right)+\operatorname{TV}\left(q_{t_0} \hat{P}_{\mathrm{flow}}^{t_0,N_{\text {pred }}} \hat{P}_{\mathrm{ULMC}}^{t_0+T_{\text{pred}},N_{\text {corr }}}, q_{t_0+T_{\text {pred }}}\right) \\
& \leq \mathrm{TV}\left(p, q_{t_0}\right)+\operatorname{TV}\left(q_{t_0} \hat{P}_{\mathrm{flow}}^{t_0,N_{\text {pred }}} \hat{P}_{\mathrm{ULMC}}^{t_0+T_{\text{pred}},N_{\text {corr }}}, q_{t_0+T_{\text {pred }}}\right)
\end{align*}

We use Lemma \ref{lem:finalcorrector} to bound the second term and choose $T_{\text{corr}}\asymp1/\sqrt{L_{S,X}}$:
\begin{align*}
    &\operatorname{TV}\left(q_{t_0} \hat{P}_{\mathrm{flow}}^{t_0,N_{\text {pred }}} \hat{P}_{\mathrm{ULMC}}^{t_0+T_{\text{pred}},N_{\text {corr }}}, q_{t_0+T_{\text {pred }}}\right)\\&\lesssim \sqrt{L_{S,X}}W_2(q_{t_0} \hat{P}_{\mathrm{flow}}^{t_0,N_{\text {pred }}},q_{t_0+T_{\text {pred }}})+   \sqrt{L_{S,X}d} h_{\mathrm{corr}} + \frac{\varepsilon_{sc}}{\sqrt{L_{S,X}}\delta}\,.
\end{align*}

Then we use Lemma \ref{lem:predict-n-step} to bound $W_2(q_{t_0} \hat{P}_{\mathrm{flow}}^{t_0,N_{\text {pred }}},q_{t_0+T_{\text {pred }}})$ and choose $T_{\text{pred}}\asymp1/\sqrt{L_{v,X}}$:
\begin{align*}
W_2(q_{t_0} \hat{P}_{\mathrm{flow}}^{t_0,N_{\text {pred }}},q_{t_0+T_{\text {pred }}})\lesssim\frac{h_{\text{pred}}}{L_{v,X}}\frac{R^2(R \vee \sqrt{d})}{\delta^4} + \frac{\varepsilon_{sc}}{L_{v,X}}
\end{align*}

Combined with the above bounds, we have that
\begin{align*}
&\operatorname{TV}\left(p \hat{P}_{\mathrm{flow}}^{t_0,N_{\text {pred }}} \hat{P}_{\mathrm{ULMC}}^{t_0+T_{\text{pred}},N_{\text {corr }}}, q_{t_0+T_{\text {pred }}}\right)\\&\lesssim \operatorname{TV}(p,q_0) + \sqrt{L_{S,X}}\left(\frac{h_{\text{pred}}}{L_{v,X}}\frac{D^2(D \vee \sqrt{d})}{\delta^4} + \frac{\varepsilon_{sc}}{L_{v,X}}\right)+   \sqrt{L_{S,X}d} h_{\mathrm{corr}} + \frac{\varepsilon_{sc}}{\sqrt{L_{S,X}}\delta}
\end{align*}

Rearrange the terms, we have that
\begin{align*}
&\operatorname{TV}\left(p \hat{P}_{\mathrm{flow}}^{t_0,N_{\text {pred }}} \hat{P}_{\mathrm{ULMC}}^{t_0+T_{\text{pred}},N_{\text {corr }}}, q_{t_0+T_{\text {pred }}}\right)\\&\lesssim \operatorname{TV}(p,q_0) + \sqrt{L_{S,X}}\frac{R^2(R \vee \sqrt{d})}{L_{v,X}\delta^4}h_{\text{pred}}+   \sqrt{L_{S,X}d} h_{\mathrm{corr}} + (\frac{1}{\sqrt{L_{S,X}}\delta}+\frac{\sqrt{L_{S,X}}}{L_{v,X}})\varepsilon_{sc}\,.
\end{align*}
\end{proof}

\thmmaintheoremreversePFODE*

\begin{proof}
Iterating Lemma \ref{lem:finallemma} for $N'$ times from the initial distribution $q_0\sim \mathcal{N}(0,\mathrm{I}_d)$, where $N' = \frac{1-\delta}{T_{\mathrm{pred}}}=(1-\delta)L_{v,X}\leq L_{v,X} $.
By induction, we have the following guarantee:
\begin{align*}
&\operatorname{TV}(q_{1-\delta}^{\operatorname{ULMC}},p_{1-\delta})\\ &\lesssim N'\left(\sqrt{L_{S,X}}\frac{R^2(R \vee \sqrt{d})}{L_{v,X}\delta^4}h_{\text{pred}}+   \sqrt{L_{S,X}d} h_{\mathrm{corr}} + (\frac{1}{\sqrt{L_{S,X}}\delta}+\frac{\sqrt{L_{S,X}}}{L_{v,X}})\varepsilon_{sc}\right)\\&\lesssim L_{v,X}\left(\sqrt{L_{S,X}}\frac{R^2(R \vee \sqrt{d})}{L_{v,X}\delta^4}h_{\text{pred}}+   \sqrt{L_{S,X}d} h_{\mathrm{corr}} + (\frac{1}{\sqrt{L_{S,X}}\delta}+\frac{\sqrt{L_{S,X}}}{L_{v,X}})\varepsilon_{sc}\right)\\
&\lesssim \frac{R^3(R \vee \sqrt{d})}{\delta^6}h_{\mathrm{pred}}+\frac{R^3}{\delta^5}h_{\mathrm{corr}} + \frac{\varepsilon_{sc}}{\delta^2}\,.
\end{align*}
Then, we finish the proof.


\end{proof}

\section{The Proof for One-step RF-based Models}

\thminstaflowresults*
\begin{proof}
As shown in \citet{lyu2023convergenceconsistencyv1}, the error term can be decomposed into the following form:
\begin{align*}
W_2\left(f_{\theta, 0} \sharp \mathcal{N}\left(0,  I_d\right), q^*\right)&\leq W_2(f_{\theta, 0} \sharp \mathcal{N}\left(0,  I_d\right),q_{1-\delta})+ W_2(q_{1-\delta},q^*)\,.
\end{align*}
For the second term, by using lemma \ref{lem:orderofdelta}, we know that it is smaller than $(\sqrt{d}+R)\delta$. Hence, we focus on the first discretization term. Let $q_{t_0}\sim\mathcal{N}(0,I_d)$.
Considering the property of $f^{v^*}$, we have the following inequality:
\begin{align*}
&\left( \mathbb{E}_{X_{t_0} \sim q_{t_0}} \left[ \left\| f_\theta(X_{t_0}, t_0) - f^{v^*}(X_{t_0}, t_0) \right\|_2^2 \right] \right)^{1/2} \\
&= \left( \mathbb{E}_{X_{t_0} \sim q_{t_0}} \left[ \left\| \sum_{k=0}^{k-1} \left( f_\theta(X_{t_{k}}, t_{k}) - f_\theta(X_{t_{k+1}}, t_{k+1}) \right) \right\|_2^2 \right] \right)^{1/2} \\
&\leq \sum_{k=0}^{K-1} \left( \mathbb{E}_{X_{t_k} \sim q_{t_k}} \left[ \left\|  f_\theta(X_{t_{k}}, t_{k}) - f_\theta(X_{t_{k+1}}, t_{k+1}) \right\|_2^2 \right] \right)^{1/2} \\
&= \sum_{k=0}^{K-1} \left( \mathbb{E}_{X_{t_k} \sim q_{t_k}} \left[ \left\| f_\theta(X_{t_{k}}, t_{k}) - f_\theta(\hat{X}^\phi_{t_{k+1}}, t_{k+1}) + f_\theta(\hat{X}^\phi_{t_{k+1}}, t_{k+1}) - f_\theta(X_{t_{k+1}}, t_{k+1}) \right\|_2^2 \right] \right)^{1/2} \\
&\leq \sum_{k=0}^{K-1} \left( \mathbb{E}_{X_{t_k} \sim q_{t_k}} \left[ \left\| f_\theta(X_{t_{k}}, t_{k}) - f_\theta(\hat{X}^\phi_{t_{k+1}}, t_{k+1}) \right\|_2^2 \right] \right)^{1/2} \\
&\quad + \sum_{k=0}^{K-1} \left( \mathbb{E}_{X_{t_k} \sim q_{t_k}} \left[ \left\| f_\theta(\hat{X}^\phi_{t_{k+1}}, t_{k+1}) - f_\theta(X_{t_{k+1}}, t_{k+1}) \right\|_2^2 \right] \right)^{1/2} := E_1 + E_2.
\end{align*}

By using \cref{ass:onestepfmatcherror}, we can bound $E_1$ by:
\begin{align*}
&E_1 = \sum_{k=0}^{K-1} \left( \mathbb{E}_{X_{t_k} \sim q_{t_k}} \left[ \left\| f_\theta(X_{t_{k}}, t_{k}) - f_\theta(\hat{X}^\phi_{t_{k+1}}, t_{k+1}) \right\|_2^2 \right]^{1/2} \right) \\
&\leq \varepsilon_{\mathrm{cm}} \sum_{k=0}^{K-1} h_k = \varepsilon_{\mathrm{cm}} (t_K - t_0)\leq\varepsilon_{\mathrm{cm}}
\end{align*}

For the $E_2$ term, we can use \cref{ass:flip} to decompose it:
\begin{align*}
E_2 &= \sum_{k=0}^{K-1} \left( \mathbb{E}_{X_{t_k} \sim p_{t_k}} \left[ \left\| f_\theta(\hat{X}^\phi_{t_{k+1}}, t_{k+1}) - f_\theta(X_{t_{k+1}}, t_{k+1}) \right\|_2^2 \right]^{1/2} \right) \\
&\leq \sum_{k=0}^{K-1} L_f \left( \mathbb{E}_{X_{t_k} \sim p_{t_k}} \left[ \left\| \hat{X}^\phi_{t_{k+1}} - X_{t_{k+1}} \right\|_2^2 \right]^{1/2} \right) \\
\end{align*}

The next step is to bound $\mathbb{E}\left[\| \hat{X}^\phi_{t_{k+1}} - X_{t_{k+1}} \|\right]^2$, which is the one-step predictor error. We can use  \cref{running one step} in lemma \ref{prediction step} to bound this term:
\begin{equation*}
    \mathbb{E}[\|X_{t_{k+1}} - \hat{X}_{t_{k+1}}\|^2] \lesssim  (t_{k+1}-t_{k})^4 \frac{R^4(R^2 \vee d)}{(1-t_k)^8} + (t_{k+1}-t_{k})^2 \varepsilon^2_{sc}\,,
\end{equation*}
which indicates
\begin{align*}
E_2 &= \sum_{k=0}^{K-1} \left( \mathbb{E}_{X_{t_k} \sim q_{t_k}} \left[ \left\| f_\theta(\hat{X}^\phi_{t_{k+1}}, t_{k+1}) - f_\theta(X_{t_{k+1}}, t_{k+1}) \right\|_2^2 \right]^{1/2} \right) \\
&\leq \sum_{k=0}^{K-1} L_f \left( \mathbb{E}_{X_{t_k} \sim q_{t_k}} \left[ \left\| \hat{X}^\phi_{t_{k+1}} - X_{t_{k+1}} \right\|_2^2 \right]^{1/2} \right) \\
&\lesssim L_f R^2(R \vee \sqrt{d}) \sum_{k=0}^{K-1} \frac{(t_{k+1}-t_k)^2}{(1-t_k)^4}+ L_f\varepsilon_{sc} \leq L_f\frac{R^2(R \vee \sqrt{d})}{\delta^2}\sum_{k=0}^{K-1} \frac{(t_{k+1}-t_k)^2}{(1-t_k)^2} + L_f\varepsilon_{sc}\,.
\end{align*}
When considering the EDM stepsize, we know that 
\begin{align*}
    1-t_k=(\delta+k h)^a, h=\frac{T^{\frac{1}{a}}-\delta}{K}\,,
\end{align*}
which indicates that $\frac{h_k}{h} \asymp (1-t_k)^{\frac{a-1}{a}}$. Then, we know that 
\begin{align*}
\sum_{k=1}^K \frac{h_k^2}{(1-t_k)^2} \asymp h \sum_{k=1}^K \frac{h_k}{t_k^{\frac{a+1}{a}}} \asymp h \int_\delta^T \frac{1}{(1-t)^{\frac{a+1}{a}}} \mathrm{~d} t \asymp h \delta^{-\frac{1}{a}} \asymp \frac{(1 / \delta)^{\frac{1}{a}}}{K} .    
\end{align*}

Combining $E_1$ and $E_2$, we have:
\begin{align*}
\left( \mathbb{E}_{X_{t_0} \sim q_{t_0}} \left[ \left\| f_\theta(X_{t_0}, t_0) - f(X_{t_0}, t_0) \right\|_2^2 \right] \right)^{1/2}\lesssim \frac{L_fR^2(R \vee \sqrt{d})}{\delta^{2+1/a}K} + L_f\varepsilon_{sc}+ \varepsilon_{cm}
\end{align*}


\begin{align*}
W_2\left(f_{\theta, 0} \sharp \mathcal{N}\left(0,  I_d\right), q^*\right)&\leq W_2(f_{\theta, 0} \sharp \mathcal{N}\left(0,  I_d\right),q_{1-\delta})+ W_2(q_{1-\delta},q^*)\\
&\lesssim \frac{L_fR^2(R \vee \sqrt{d})}{\delta^{2+1/a}K} + L_f\varepsilon_{sc}+ \varepsilon_{\mathrm{cm}}+(\sqrt{d}+R)\delta\,.
\end{align*}

So the sample complexity is $ K\lesssim \frac{L_fR^2(\sqrt{d}+R)^4(R \vee \sqrt{d})}{\epsilon_{W_2}^{3+1/a}} $.

\end{proof}
\bibliographystyle{plainnat}
\bibliography{ref.bib}

\begin{thebibliography}{30}
\providecommand{\natexlab}[1]{#1}
\providecommand{\url}[1]{\texttt{#1}}
\expandafter\ifx\csname urlstyle\endcsname\relax
  \providecommand{\doi}[1]{doi: #1}\else
  \providecommand{\doi}{doi: \begingroup \urlstyle{rm}\Url}\fi

\bibitem[Benton et~al.(2023)Benton, De~Bortoli, Doucet, and Deligiannidis]{benton2023linear}
Joe Benton, Valentin De~Bortoli, Arnaud Doucet, and George Deligiannidis.
\newblock Linear convergence bounds for diffusion models via stochastic localization.
\newblock \emph{arXiv preprint arXiv:2308.03686}, 2023.

\bibitem[Chen et~al.(2024)Chen, Ren, Ying, and Rotskoff]{yinuochen2024accelerating}
Haoxuan Chen, Yinuo Ren, Lexing Ying, and Grant~M Rotskoff.
\newblock Accelerating diffusion models with parallel sampling: Inference at sub-linear time complexity.
\newblock \emph{arXiv preprint arXiv:2405.15986}, 2024.

\bibitem[Chen et~al.(2022)Chen, Chewi, Li, Li, Salim, and Zhang]{chen2022sampling}
Sitan Chen, Sinho Chewi, Jerry Li, Yuanzhi Li, Adil Salim, and Anru~R Zhang.
\newblock Sampling is as easy as learning the score: theory for diffusion models with minimal data assumptions.
\newblock \emph{arXiv preprint arXiv:2209.11215}, 2022.

\bibitem[Chen et~al.(2023{\natexlab{a}})Chen, Chewi, Lee, Li, Lu, and Salim]{chen2023probability}
Sitan Chen, Sinho Chewi, Holden Lee, Yuanzhi Li, Jianfeng Lu, and Adil Salim.
\newblock The probability flow ode is provably fast.
\newblock \emph{arXiv preprint arXiv:2305.11798}, 2023{\natexlab{a}}.

\bibitem[Chen et~al.(2023{\natexlab{b}})Chen, Daras, and Dimakis]{chen2023restoration}
Sitan Chen, Giannis Daras, and Alexandros~G Dimakis.
\newblock Restoration-degradation beyond linear diffusions: A non-asymptotic analysis for ddim-type samplers.
\newblock \emph{arXiv preprint arXiv:2303.03384}, 2023{\natexlab{b}}.

\bibitem[De~Bortoli(2022)]{de2022convergence}
Valentin De~Bortoli.
\newblock Convergence of denoising diffusion models under the manifold hypothesis.
\newblock \emph{arXiv preprint arXiv:2208.05314}, 2022.

\bibitem[Dou et~al.(2024)Dou, Chen, Wang, and Yang]{doutheoryconsistency}
Zehao Dou, Minshuo Chen, Mengdi Wang, and Zhuoran Yang.
\newblock Theory of consistency diffusion models: Distribution estimation meets fast sampling.
\newblock In \emph{Forty-first International Conference on Machine Learning, {ICML} 2024, Vienna, Austria, July 21-27, 2024}. OpenReview.net, 2024.

\bibitem[Esser et~al.(2024)Esser, Kulal, Blattmann, Entezari, M{\"u}ller, Saini, Levi, Lorenz, Sauer, Boesel, et~al.]{esser2024scalingsd3}
Patrick Esser, Sumith Kulal, Andreas Blattmann, Rahim Entezari, Jonas M{\"u}ller, Harry Saini, Yam Levi, Dominik Lorenz, Axel Sauer, Frederic Boesel, et~al.
\newblock Scaling rectified flow transformers for high-resolution image synthesis.
\newblock In \emph{Forty-first International Conference on Machine Learning}, 2024.

\bibitem[Gao and Zhu(2024)]{gao2024convergencediffusion}
Xuefeng Gao and Lingjiong Zhu.
\newblock Convergence analysis for general probability flow odes of diffusion models in wasserstein distances.
\newblock \emph{arXiv preprint arXiv:2401.17958}, 2024.

\bibitem[Gao et~al.(2023)Gao, Huang, and Jiao]{gao2023gaussianflow}
Yuan Gao, Jian Huang, and Yuling Jiao.
\newblock Gaussian interpolation flows.
\newblock \emph{arXiv preprint arXiv:2311.11475}, 2023.

\bibitem[Gao et~al.(2024)Gao, Huang, Jiao, and Zheng]{gao2024convergence}
Yuan Gao, Jian Huang, Yuling Jiao, and Shurong Zheng.
\newblock Convergence of continuous normalizing flows for learning probability distributions.
\newblock \emph{arXiv preprint arXiv:2404.00551}, 2024.

\bibitem[Go et~al.(2024)Go, Park, Jang, Kim, Kwon, and Kim]{go2024splatflow3dv2}
Hyojun Go, Byeongjun Park, Jiho Jang, Jin-Young Kim, Soonwoo Kwon, and Changick Kim.
\newblock Splatflow: Multi-view rectified flow model for 3d gaussian splatting synthesis.
\newblock \emph{arXiv preprint arXiv:2411.16443}, 2024.

\bibitem[Karras et~al.(2022)Karras, Aittala, Aila, and Laine]{karras2022elucidatingvesde}
Tero Karras, Miika Aittala, Timo Aila, and Samuli Laine.
\newblock Elucidating the design space of diffusion-based generative models.
\newblock \emph{arXiv preprint arXiv:2206.00364}, 2022.

\bibitem[Kim et~al.(2021)Kim, Shin, Song, Kang, and Moon]{kim2021softepsilonl}
Dongjun Kim, Seungjae Shin, Kyungwoo Song, Wanmo Kang, and Il-Chul Moon.
\newblock Soft truncation: A universal training technique of score-based diffusion model for high precision score estimation.
\newblock \emph{arXiv preprint arXiv:2106.05527}, 2021.

\bibitem[Li et~al.(2023)Li, Wei, Chen, and Chi]{li2023towards}
Gen Li, Yuting Wei, Yuxin Chen, and Yuejie Chi.
\newblock Towards faster non-asymptotic convergence for diffusion-based generative models.
\newblock \emph{arXiv preprint arXiv:2306.09251}, 2023.

\bibitem[Li et~al.(2024{\natexlab{a}})Li, Huang, and Wei]{li2024towardsconsistencyv2}
Gen Li, Zhihan Huang, and Yuting Wei.
\newblock Towards a mathematical theory for consistency training in diffusion models.
\newblock \emph{arXiv preprint arXiv:2402.07802}, 2024{\natexlab{a}}.

\bibitem[Li et~al.(2024{\natexlab{b}})Li, Wei, Chi, and Chen]{li2024sharp}
Gen Li, Yuting Wei, Yuejie Chi, and Yuxin Chen.
\newblock A sharp convergence theory for the probability flow odes of diffusion models.
\newblock \emph{arXiv preprint arXiv:2408.02320}, 2024{\natexlab{b}}.

\bibitem[Li et~al.(2024{\natexlab{c}})Li, Chu, Shi, and Lin]{li2024flowdreamer3dv1}
Hangyu Li, Xiangxiang Chu, Dingyuan Shi, and Wang Lin.
\newblock Flowdreamer: Exploring high fidelity text-to-3d generation via rectified flow.
\newblock \emph{arXiv preprint arXiv:2408.05008}, 2024{\natexlab{c}}.

\bibitem[Liu et~al.(2022)Liu, Gong, and Liu]{liu2022flowfirst}
Xingchao Liu, Chengyue Gong, and Qiang Liu.
\newblock Flow straight and fast: Learning to generate and transfer data with rectified flow.
\newblock \emph{arXiv preprint arXiv:2209.03003}, 2022.

\bibitem[Liu et~al.(2023)Liu, Zhang, Ma, Peng, et~al.]{liu2023instaflow}
Xingchao Liu, Xiwen Zhang, Jianzhu Ma, Jian Peng, et~al.
\newblock Instaflow: One step is enough for high-quality diffusion-based text-to-image generation.
\newblock In \emph{The Twelfth International Conference on Learning Representations}, 2023.

\bibitem[Lyu et~al.(2024)Lyu, Chen, and Feng]{lyu2023convergenceconsistencyv1}
Junlong Lyu, Zhitang Chen, and Shoubo Feng.
\newblock Sampling is as easy as keeping the consistency: convergence guarantee for consistency models.
\newblock In \emph{Forty-first International Conference on Machine Learning}, 2024.

\bibitem[Silveri et~al.(2024)Silveri, Conforti, and Durmus]{silveri2024theoreticalflowkl}
Marta~Gentiloni Silveri, Giovanni Conforti, and Alain Durmus.
\newblock Theoretical guarantees in kl for diffusion flow matching.
\newblock \emph{arXiv preprint arXiv:2409.08311}, 2024.

\bibitem[Song et~al.(2020)Song, Sohl-Dickstein, Kingma, Kumar, Ermon, and Poole]{song2020sde}
Yang Song, Jascha Sohl-Dickstein, Diederik~P Kingma, Abhishek Kumar, Stefano Ermon, and Ben Poole.
\newblock Score-based generative modeling through stochastic differential equations.
\newblock \emph{arXiv preprint arXiv:2011.13456}, 2020.

\bibitem[Song et~al.(2023)Song, Dhariwal, Chen, and Sutskever]{song2023consistency}
Yang Song, Prafulla Dhariwal, Mark Chen, and Ilya Sutskever.
\newblock Consistency models.
\newblock \emph{arXiv preprint arXiv:2303.01469}, 2023.

\bibitem[Wang et~al.(2024{\natexlab{a}})Wang, Yang, Huang, Wang, and Li]{wang2024rectified}
Fu-Yun Wang, Ling Yang, Zhaoyang Huang, Mengdi Wang, and Hongsheng Li.
\newblock Rectified diffusion: Straightness is not your need in rectified flow.
\newblock \emph{arXiv preprint arXiv:2410.07303}, 2024{\natexlab{a}}.

\bibitem[Wang et~al.(2024{\natexlab{b}})Wang, Guo, Huang, Huang, Wang, You, Li, and Zhao]{wang2024frierenvideo}
Yongqi Wang, Wenxiang Guo, Rongjie Huang, Jiawei Huang, Zehan Wang, Fuming You, Ruiqi Li, and Zhou Zhao.
\newblock Frieren: Efficient video-to-audio generation with rectified flow matching.
\newblock \emph{arXiv preprint arXiv:2406.00320}, 2024{\natexlab{b}}.

\bibitem[Yang et~al.(2024)Yang, Wang, Jiang, and Li]{yang2024leveraging}
Ruofeng Yang, Zhijie Wang, Bo~Jiang, and Shuai Li.
\newblock Leveraging drift to improve sample complexity of variance exploding diffusion models.
\newblock In \emph{The Thirty-eighth Annual Conference on Neural Information Processing Systems}, 2024.
\newblock URL \url{https://openreview.net/forum?id=euQ0C4iS7O}.

\bibitem[Yang et~al.(2025{\natexlab{a}})Yang, Jiang, Chen, and Li]{yang2025improved}
Ruofeng Yang, Bo~Jiang, Cheng Chen, and Shuai Li.
\newblock Improved discretization complexity analysis of consistency models: Variance exploding forward process and decay discretization scheme.
\newblock In \emph{Forty-second International Conference on Machine Learning}, 2025{\natexlab{a}}.
\newblock URL \url{https://openreview.net/forum?id=CZTcRSxkfe}.

\bibitem[Yang et~al.(2025{\natexlab{b}})Yang, Jiang, and Li]{yang2025polynomial}
Ruofeng Yang, Bo~Jiang, and Shuai Li.
\newblock The polynomial iteration complexity for variance exploding diffusion models: Elucidating sde and ode samplers.
\newblock In \emph{The 28th International Conference on Artificial Intelligence and Statistics}, 2025{\natexlab{b}}.

\bibitem[Zhang and Zou(2024)]{zhang2024collapse}
Yi~Zhang and Difan Zou.
\newblock On the collapse errors induced by the deterministic sampler for diffusion models.
\newblock In \emph{NeurIPS 2024 Workshop: Self-Supervised Learning-Theory and Practice}, 2024.

\end{thebibliography}






\end{document}